\documentclass{article}

\usepackage{caption}
\usepackage{subcaption}
\usepackage{enumitem}
\usepackage{xcolor}
\usepackage{arxiv}

\usepackage[utf8]{inputenc} 
\usepackage[T1]{fontenc}    
\usepackage{hyperref}       
\usepackage{url}            
\usepackage{booktabs}       
\usepackage{amsfonts}       
\usepackage{nicefrac}       
\usepackage{microtype}      
\usepackage{lipsum}		
\usepackage{graphicx}
\usepackage{natbib}
\usepackage{doi}
\newenvironment{proofsketch}[1][Proof sketch]{\begin{proof}[#1]}{\end{proof}}
\usepackage{caption}
\usepackage{subcaption}
\usepackage{graphicx}
\usepackage{amssymb}
\usepackage{algorithm}
\usepackage{algpseudocode}

\usepackage{caption}

 \usepackage{amsthm}

\usepackage{algpseudocode}
\usepackage{statmath}

\newtheorem{theorem}{Theorem}
\newtheorem{corollary}{Corollary}
\newtheorem{lemma}{Lemma}

\newtheorem{assumption}{Assumption}[section]
\newtheorem{remark}{Remark}[section]

 \usepackage{amsmath}

\title{	
A Theoretical Framework for LLM Fine-tuning Using Early Stopping for Non-random Initialization}

\author{ Zexuan Sun \\
	 Department of Statistics\\
  University of Wisconsin, Madison\\
  Madison, WI 53706 \\
  \texttt{zexuan.sun@wisc.edu} \\
	\And
	 Garvesh Raskutti \\
  Department of Statistics\\
  University of Wisconsin, Madison\\
  Madison, WI 53706 \\
  \texttt{raskutti@stat.wisc.edu} \\
}

\renewcommand{\shorttitle}{}

\hypersetup{
pdftitle={A template for the arxiv style},
pdfsubject={q-bio.NC, q-bio.QM},
pdfauthor={David S.~Hippocampus, Elias D.~Striatum},
pdfkeywords={First keyword, Second keyword, More},
}

\begin{document}
\maketitle
\begin{abstract}
	 In the era of large language models (LLMs), fine-tuning pretrained models has become ubiquitous. Yet the theoretical underpinning remains an open question. A central question is why only a few epochs of fine-tuning are typically sufficient to achieve strong performance on many different tasks. In this work, we approach this question by developing a statistical framework, combining rigorous early stopping theory with the attention-based Neural Tangent Kernel (NTK) for LLMs, offering  new theoretical insights on fine-tuning practices. 
     Specifically, we formally extend  classical NTK theory \citep{jacot} to non-random (i.e., pretrained) initializations and provide a convergence guarantee for attention-based fine-tuning. 
     One key insight provided by the theory is that the convergence rate with respect to sample size is closely linked to the eigenvalue decay rate of the empirical kernel matrix induced by the NTK.
     We also demonstrate how the  framework can be used to explain task vectors for multiple tasks in LLMs. Finally, experiments with modern language models on real-world datasets provide empirical evidence supporting our theoretical insights.
\end{abstract}


\section{Introduction}
With the success of modern large language models (LLMs) (e.g. GPT 
models \citep{ radford2018gpt, brown2020gpt3}), it has become standard practice to address supervised natural language processing (NLP) tasks such as topic classification, textual entailment and others by fine-tuning a pretrained model on a task-specific dataset \citep{llmref1,llmref2}. Despite the strong empirical performance of this approach, our theoretical understanding of fine-tuning remains limited. A widely adopted convention in practice is to initialize with parameters from a pre-trained model and then do gradient updates for only a small number of epochs, typically between two and five, since empirical evidence suggests that longer training yields diminishing returns and may even increase the risk of overfitting \citep{overfit}. This naturally raises the question: why is a few epochs of fine-tuning sufficient? Intuitively, one might argue that pretrained models, having been exposed to vast text corpora, already provide strong representations. Yet a more formal and rigorous interpretation of this phenomenon is still lacking.

The vast number of parameters in modern models motivates the application of Neural Tangent Kernel (NTK) theory \citep{jacot} to the study of fine-tuning \citep{kernellm}. However, NTK theory alone is insufficient to capture the ``early-stopping'' phenomenon that is central to modern practice because it does not explain how changes in the data are captured by the NTK. To address this limitation, we build on theoretical advances in early stopping \citep{esnonpara,sunaistats} and apply to the LLM setting. In particular, \cite{sunaistats} combine early stopping with NTK theory to characterize the performance of wide fully connected neural networks and gradient boosting decision trees in estimating variable importance. This line of work suggests promising extensions to more complex network architectures in the context of modern fine-tuning.

In this paper, we consider a transformer-based network for \textit{regression} learning problems to give a unique perspective of understanding modern fine-tuning for large models. Our major contributions are:
\begin{enumerate}

\item We formally extend conventional infinite-width analysis to a multi-head attention-based network and our theory accounts for a pretrained initialization, which includes the stability of  NTK during training and the linearization of the multi-head attention network (Lemma \ref{linlemma}) in our setup. 
\item We extend the theoretical framework in \cite{sunaistats} to transformer based network fine-tuning, which includes the setting where additional structure can be appended to certain hidden layers of the pretrained model (Theorem   \ref{ftthm} and  \ref{lhthm}). 


\item  Based on the proven theoretical results, we  give formal explanation  of task 
arithmetic \citep{taskvector}, including negation and addition. And we recover the same weight disentanglement condition in \cite{tangenttask}, which is claimed to be the only necessary condition for task vector to work (Section \ref{addsec}) . 

\item We discuss the valuable insights offered by our proposed theoretical framework and connect them to real-world LLM fine-tuning practices (Section \ref{practicalinsight}). We also conduct experiments with modern LLMs, specifically, the GPT-Neo model, to provide empirical evidence supporting our theoretical insights (Section \ref{mainexp}).
\end{enumerate}


\section{RELATED WORK}

\paragraph{Early stopping theory.}
Early stopping has a long history \citep{esold1, esold2}, with significant theoretical developments in classification boosting \citep{esboost1, esbtyu, esboost2}, L2-boosting \citep{bl2, l2boost2}, and gradient-based algorithms in reproducing kernel Hilbert spaces (RKHS) \citep{grdrkhs1, esboost2, esnonpara, esbs, sunaistats}. Building on this line of work, the present paper extends the theoretical framework of early stopping in \cite{sunaistats}, which focused on large fully connected networks, to more complex neural architectures. To the best of our knowledge, this work provides the first theoretical analysis that leverages early stopping theory in the context of fine-tuning.

\paragraph{Kernel view of neural networks.}
Overparameterized models, including deep neural networks, are often analyzed through the Neural Tangent Kernel (NTK) framework \citep{jacot}, which provides guarantees of global convergence \citep{ode, ntkre1, ntkre2, ntkre3} and justifies linear approximations of network behavior \citep{linearnn}, with extensions to convolutional architectures \citep{cntk2}. Subsequent work based on the Tensor Program framework generalized these results to a wide range of neural architectures \citep{yang2019wide,yangtraning,yang2020tensor2}. However, these analyses all rely on random Gaussian initialization of the weights. \cite{kernellm} extend the theory to non-random (i.e., pretrained) initialization under the assumption that the pretraining task is sufficiently related to the downstream task. Specifically, they assume that the gradient of the pretrained network with respect to the downstream task tends to zero as the width approaches infinity. In contrast, our work do not have this assumption and establishes similar theoretical guarantees under sufficiently large width and embedding dimension, relying on different proof techniques.

\paragraph{Task vector theory.} 
Proposed by \cite{taskvector}, task arithmetic has recently emerged as a cost-effective and scalable method for editing pretrained models directly in weight space. Subsequent work has sought to develop a theoretical understanding of this approach, for example by analyzing it through the lens of NTK theory \citep{tangenttask} or by considering a conceptual learning setting in which each task corresponds to a binary classification problem defined by a discriminative pattern \citep{taskvthm}. Notably, the explanation derived from our theoretical results aligns with the \textit{weight disentanglement} condition introduced in \citep{tangenttask}.





\section{BASIC SETUP AND PRELIMS}
\subsection{Statistical framework}

In order to adapt to multiple scenarios, we consider a general learning setup.
Assume that initially, we have source data $\left(\mathbf{X}^{\text{src}}_i, Y^{\text{src}}_i\right), i=1, \ldots, N_1$ from $(X^{\text{src}}, Y^{\text{src}}) \sim P_{\text{src}}$ to train a pretrained model $f_{\text{PT}}$. Next, we have data from our target $\left(\mathbf{X}^{\text{tgt}}_i, Y^{\text{tgt}}_i\right), i=1, \ldots, N_1$ from $(X^{\text{tgt}}, Y^{\text{tgt}}) \sim P_{\text{tgt}}$. The basic idea is using the pretrained model $f_{\text{PT}}$ as a starting point to model our target function $f_{\text{tgt}}$ without the need to train a new model from scratch. More precisely, we plan to use $f_{\text{PT}}$ as initialization and train with $\left(\mathbf{X}^{\text{tgt}}_i, Y^{\text{tgt}}_i\right)$ to get the so-called fine-tuned model $f_{\text{FT}}$. 
We use
$\boldsymbol{X}^{\text{src}}$,
$\boldsymbol{X}^{\text{tgt}}$, 
 $\boldsymbol{Y}^{\text{src}}$
and $\boldsymbol{Y}^{\text{tgt}}$ to denote $({\mathbf{X}^{\text{src}}_1}^{ T}, \dots, {\mathbf{X}^{\text{src}}_{N_1}}^{ T})$, $({\mathbf{X}^{\text{tgt}}_1}^{ T}, \dots, {\mathbf{X}^{\text{tgt}}_{N_2}}^{ T})$,
$(Y^{\text{src}}_1, \dots, Y^{\text{src}}_{N_1})^T$
and $(Y^{\text{tgt}}_1, \dots, Y^{\text{tgt}}_{N_2})^T$, respectively.
 




We assume that the samples take the following form with respect to data $X^{\text{src}}$ and $X^{\text{tgt}}$ :
\begin{equation}
    \begin{aligned}
& Y^{\text{src}}_i=
\mathbb{E}[Y_i^{\text{src}} \mid X_i^{\text{src}}]
+w^{\text{src}}_i \\
& Y^{\text{tgt}}_i=\mathbb{E}\left[Y_i^{\text{tgt}} \mid X_i^{\text{tgt}} \right]+w_i^{\text{tgt}}
\end{aligned}
\end{equation}
where $w^{\text{src}}_i$ and $w^{\text{tgt}}_i$ are independent noise random variables. The true functions $f_{\text{PT}}(X^{\text{src}}) = \mathbb{E}[Y^{\text{src}} \mid X^{\text{src}}]$ and $f_{\text{tgt}}(X^{\text{tgt}}) = \mathbb{E}[Y^{\text{tgt}} \mid X^{\text{tgt}}]$.



To assess how well the approximation $\widehat{f_{\text{tgt}}}$ matches the target function $f_{\text{tgt}}$, we mainly consider  $L^2(P)$-norm:
\begin{equation}
    \left\|\widehat{f_{\text{tgt}}}-f_{\text{tgt}}\right\|_2^2=\mathbb{E}\left[\left(\widehat{f_{\text{tgt}}}(\mathbf{X^{\text{tgt}}})-f_{\text{tgt}}(\mathbf{X^{\text{tgt}}})\right)^2\right]
\end{equation}
which corresponds to the commonly used mean squared error.

\subsection{Network Structure}
\label{nnsetup}
In this paper, we consider network structure with basic transformer blocks\footnote{Specifically, we consider a vanilla transformer decoder without  positional encoding, but layer normalization is considered in the theory.}.   For theoretical convenience, we formulate the learning task as a \textit{regression} problem.


\paragraph{Neural networks:}  Let $f^{l}(x)$ denotes the output of $l^{\text {th }}$ layer for an input $x \in \mathcal{X} \subset \mathbb{R}^{T \times d^0}$, and $g^{l}(x):=\phi\left(f^{l}(x)\right)$ the corresponding post-nonlinearity where $\phi: \mathbb{R} \rightarrow \mathbb{R}$ is the activation function. We assume the network has $L$ hidden layers, making the output $f^{L+1}(x) \in \mathbb{R}$. If $l^{\text {th }}$  layer is an attention layer, the output   $f^{l}(x)  \in \mathbb{R}^{T \times d^{l}}$ 
where $T$ is the sequence length (spatial dimension) and $d^l$ is the 
 embedding dimensions. The network is optimized via gradient descent with learning rate $\epsilon$.


\paragraph{Attention:}   The output of each head $h$ is computed as follows:
\begin{equation}
    f^{l h}(x)=\zeta^{lh}
    g^{l-1}(x) W^{l h, V} \in \mathbb{R}^{T \times d^{l}}
\end{equation}
where
$\zeta^{lh} = \zeta\left(\frac{1}{d^{l, G}} g^{l-1}(x) W^{l h, Q}\left(g^{l-1}(x) W^{l h, K}\right)^{\top}\right)$,
$W^{l h , Q}, W^{l h , K} \in \mathbb{R}^{d^{l-1} \times d^{l, G}}$, $W^{l, V} \in \mathbb{R}^{d^{l-1} \times d^{l}}$, and $\zeta$ is the row-wise softmax function. Rather than using the conventional $\frac{1}{\sqrt{d^{l,G}}}$ scaling, we consider the same scaling as in \cite{yang2019wide}. This scaling is crucial to the theoretical proof, which makes sure the Jacobian stays within the proper order as the network width increases.

Each attention layer contains $d^{l,H}$ heads. The \textit{multi-head} output $f^{l}(x)$ is given by
\begin{equation}
f^{l}(x) = \left[f^{l1}(x), \ldots, f^{l d^{l,H}}(x)\right] W^{l,O},
\end{equation}
where the outputs of the $d^{l,H}$ independently parameterized heads are concatenated into a $T \times (d^{l,H} d^{l})$ matrix and then projected back to dimension $T \times d^{l}$ via $W^{l,O} \in \mathbb{R}^{d^{l,H} d^{l} \times d^{l}}$.

\paragraph{Weight distribution:} 
 We assume random Gaussian
initialization of the weights for the \emph{pretrained model} $f_{\text{PT}}$ only.
Specifically, for $f_{\text{PT}}$, we 
initialize the weights with 
 $W_{i j}^{l h, Q} \sim \mathcal{N}\left(0, \sigma_Q^2 / d^{l-1}\right)$, $W_{ i j}^{l h, K} \sim \mathcal{N}\left(0, \sigma_K^2 / d^{l-1}\right), W_{ i j}^{l h, V} \sim \mathcal{N}\left(0, \sigma_V^2 / d^{l-1}\right)$, and $W_{ i j}^{l, O} \sim \mathcal{N}\left(0, \sigma_O^2 /\left(d^{l, H} d^{l}\right)\right)$, all i.i.d. over the $i, j$ and $l, h$.


\subsection{Reproducing Kernel Hilbert Spaces}
To analyze early stopping in kernel-based methods, we rely on the framework of reproducing kernel Hilbert spaces  \citep{rkhs}. Specifically, consider a Hilbert space $\mathcal{H} \subset L^2(P)$ consisting of functions $g: \mathcal{X} \to \mathbb{R}$ with $\|g\|_{L^2(P)}<\infty$, equipped with an inner product $\langle \cdot, \cdot \rangle_{\mathcal{H}}$ under which $\mathcal{H}$ is complete. The space $\mathcal{H}$ is an RKHS if there exists a symmetric kernel function $\mathbb{K}: \mathcal{X} \times \mathcal{X} \to \mathbb{R}{+}$ such that (a) for each $x \in \mathcal{X}$, the function $\mathbb{K}(\cdot, x)$ belongs to $\mathcal{H}$, and (b) the reproducing property holds: $f(x)=\langle f, \mathbb{K}(\cdot, x)\rangle{\mathcal{H}}$ for all $f \in \mathcal{H}$. Such a kernel function is required to be positive semidefinite.

For a specific kernel $\mathbb{K}$ on data $X$, we define the  entries of the associated \textit{empirical kernel matrix} $K$ as follows:
\begin{equation}
     K(i,j) = \frac{1}{N} \mathbb{K}(\mathbf{X}_i, \mathbf{X}_j).
\end{equation} 
As stated in \cite{sunaistats}, 
The \textit{empirical kernel matrix} is fundamental to our theoretical  analysis, providing the basis for reparameterizing the gradient update equation and deriving the optimal early stopping rule.


\subsection{Neural Tangent Kernel}

The Neural Tangent Kernel (NTK), first introduced by \cite{jacot}, serves as a theoretical framework for analyzing neural networks within the RKHS setting. 
Denote the neural network by $f(\theta, x)$, its \textit{Neural Tangent Kernel} is defined as
\begin{equation}
    \left\langle
     \nabla_{\theta} f\left(\theta, x\right),
     \nabla_{\theta} f\left(\theta, x^{\prime}\right)\right\rangle.
\end{equation}
A large body of theoretical results has been established for neural networks under random normal initialization and in the infinite-width regime. In this setting, it is well known that the NTK remains constant throughout training, and the initial kernel converges in probability to a deterministic limiting kernel (add reference later).
\citep{sunaistats} extend these results to warm-start initialization for  fully connected networks.
For the network considered in this paper, we aim to get similar results, which enables us to build a theoretical framework to explain modern deep learning fine-tuning.

In particular, at each iteration, the network $f_{\tau}$ induces a corresponding NTK, denoted by $\mathbb{K}_{\tau}$. This kernel will converge to a stationary kernel $\mathbb{K}$ as the width goes to infinity. In particular, the functional form is determined by the structure of the network, not by the input training data \citep{sunaistats}.

In this paper, we formally show that for networks including transformer decoders with pre-trained initialization, the NTK remains stable during training when the width and hidden dimension are sufficiently large. The challenging part of extension to including transformer decoder is to deal with the multi-head attention.

Denote the Jacobian of the network evaluated on the data $\boldsymbol{X}^{\text{tgt}}$ as $J(\theta) \in \mathbb{R}^{N_2 \times |\theta|}$. The \textit{empirical kernel matrix} is  can be expressed as 
\begin{equation}
    K =   \frac{1}{N_2} J(\theta)  J(\theta)^{\top}. 
\end{equation}

The calculation of gradient becomes more complicated after introducing attention.
Consider layer $l$ to be the an multi-head attention layer.
We define the following intermediate quantities:
\begin{equation}
    \begin{aligned}
& Q_h=g^{l-1} W^{lh,Q}, \quad K_h= g^{l-1}W^{lh,K}, \quad V_h=g^{l-1} W^{lh,V}, \\
& S_h=\frac{1}{d^{l,G}} Q_h K_h^{\top}, \quad P_h=\operatorname{softmax}_{\mathrm{rows}}(S_h), \quad A_h=P_h V_h,  \\
& \Lambda_h=\operatorname{block\_ diag}\left(\Lambda_{h,0}, \ldots, \Lambda_{h,T-1}\right) \in \mathbb{R}^{T^2 \times T^2}, \\ &\Lambda_{h,i}=\operatorname{diag}\left(P_{h,i:}\right)-P_{h,i:} P_{h,i:}^{\top} . 
\end{aligned}
\label{nta}
\end{equation}
Then the
Jacobian w.r.t. the  attention weights is:
\begin{equation}
    \begin{aligned}
&    J(W^{lh,V})  = {g^{l-1}}^{\top} P_h^{\top} \frac{\partial f^{L+1}}{\partial f^{lh}} {W^{lh,O}}^{\top} \\
&     J(W^{l,O})  = \left[f^{l 1}(x), \ldots, f^{l d^{l, H}}(x)\right]^T 
     \frac{\partial  f^{L+1} }{\partial f^l}  \\
     & J(W^{lh,Q}) = {g^{l-1}}^{\top} \cdot\left(\frac{1}{ d^{l,G}} \cdot
{\frac{\partial P_h}{\partial S_h}}^{\top} \cdot \frac{\partial f^{L+1}}{\partial P_h}
\cdot K^h\right) \\
& J(W^{lh,K}) ={g^{l-1}}^{\top} \cdot\left(\frac{1}{d^{l,G}} \cdot {\frac{\partial f^{L+1}}{\partial P_h}}^{\top} \cdot {\frac{\partial P_h}{\partial S_h}}^{\top}
\cdot Q^h\right)
\label{jatt}
    \end{aligned}
\end{equation}
As the expression shows, the special added structure adds up the mathematical expressions and we need to take care of them during the following theoretical
analysis.


\subsection{Early stopping with data-dependent stopping rule}
\label{stoprulesec}
Early stopping is a key regularization technique in deep learning, typically determined via a hold-out set in practice. Although the exact optimal time is intractable, data-dependent rules provide theoretical insight \citep{esnonpara, sunaistats}.

The rule depends on the \textit{empirical kernel matrix} $K \in \mathbb{R}^{N_2 \times N_2}$ from kernel $\mathbb{K}$ on $X^{\text{tgt}}$, with eigenvalues $\{\widehat{\lambda}_i\}_{i=1}^{N_2}$. Define the \textit{local empirical Rademacher complexity} \citep{esnonpara}:
\begin{equation}
    \widehat{\mathcal{R}}_K(\varrho):=\left[\frac{1}{N} \sum_{i=1}^N \min \left\{\widehat{\lambda}_i, \varrho^2\right\}\right]^{1 / 2}
    \label{localran}
\end{equation}
For noise variance $\sigma>0$, the \textit{critical empirical radius} $\widehat{\varrho}_N$ is the smallest positive solution solving
\begin{equation}
    \widehat{\mathcal{R}}_K(\varrho) \leq \frac{\varrho^2 C_{\mathcal{H}}^2
    }{2 e \sigma}
    \label{epdef}
\end{equation}
with $C_{\mathcal{H}}=\|f_{\text{PT}}-f_{\text{tgt}}\|_{\mathcal{H}}$. 
The stopping threshold $\widehat{T}_{\max}$  is well-defined, and the optimal time $\widehat{T}_{\text{op}}$ minimizes a bound on the empirical $L^2(P)$ norm for $\tau\le \widehat{T}_{\max}$ \citep{sunarxiv}.

\section{THEORETICAL GUARANTEES}
We fist consider an easier case where
the pre-trained model $f_{\text{PT}}$
 has the same architecture as the fine-tuned model 
$f_{\text{FT}}$
, including width, number of layers, and hidden dimension
and establish a convergence result for it (Section \ref{ftsec}). This theoretical result is then used to elucidate the mechanisms underlying the empirical success of modern fine-tuning methods. Subsequently, we extend the analysis to encompass more general fine-tuning regimes that permit the incorporation of additional structural components to better accommodate a range of downstream tasks (Section  \ref{addheadsec}). Finally, under appropriate additional assumptions, we utilize the proposed framework to provide a theoretical explanation of \emph{task arithmetic}~\citep{taskvector}.

\subsection{Assumptions}
\label{assumptions}

Let $\mathcal{H}$ denote the RKHS induced by the stationary kernel $\mathbb{K}$ on $X^{\text{tgt}}$, and 
denote the Hilbert norm in $\mathcal{H}$ by $\| \cdot \|_{\mathcal{H}}$. We require the following basic assumptions for our theory. 
\begin{assumption}
    $f_{\text{PT}} - f_{\text{tgt}}$ belongs to $\mathcal{H}$, i.e., $f_{\text{PT}}- f_{\text{tgt}} \in \text{span} \{ \mathbb{K}(\cdot, X^{\text{tgt}}) \}$.
    \label{inass}
\end{assumption}
This assumption is made for purely theoretical convenience avoiding  the need to  add additional  mis-specification error terms or do projections.

 \begin{assumption}
    The data $\left\{(\mathbf{X}^{\text{src}}_i, Y^{\text{src}}_i)\right\}_{i=1}^{N_1}$, and 
    $\left\{(\mathbf{X}^{\text{tgt}}_i, Y^{\text{tgt}}_i)\right\}_{i=1}^{N_2}$
    are contained in closed and bounded set in $\mathbb{R}^{p+1}$.
    \label{xb}
\end{assumption}
\begin{assumption}
     $w_i^{\text{tgt}}$ are independent zero-mean random variables satisfying the following sub-Gaussianity condition:
     \label{rmdassum}
\begin{equation}
   \mathbb{E}\left[e^{t w_i^{\text{tgt}}}\right] \leq e^{t^2 \sigma^2 / 2}, \text { for all } t \in \mathbb{R} . 
\end{equation}
\end{assumption}
It is well known that boundedness is a sufficient condition for sub-Gaussianity.
\begin{assumption}
    The dropout error does satisfies $\|\boldsymbol{Y^{\text{tgt}}} - f_{\text{PT}}(\boldsymbol{X}^{\text{tgt}}) \|_2 =  O(\sqrt{N})$.
    \label{dropb}
\end{assumption}
This condition effectively ensures that the initialization using the pre-trained model is reasonable
\begin{assumption}
    The empirical kernel matrix induced by the stationary NTK, as width and hidden dimension  $ d^{l} \rightarrow \infty$, $K$ is full rank.
  \label{linea1}  
\end{assumption}
 \begin{assumption}
 The activation function $\phi$ satisfies
\begin{equation}
    |\phi(0)|, \quad\left\|\phi^{\prime}\right\|_{\infty}, \quad \sup _{x \neq \tilde{x}}\left|\phi^{\prime}(x)-\phi^{\prime}(\tilde{x})\right| /|x-\tilde{x}|<\infty.
\end{equation}
 \end{assumption}
 \begin{assumption}
  \label{linea4}
The pretrained model 
$f_{\text{PT}}$
is trained starting from weights drawn from a normal random initialization.
\end{assumption}
     

\subsection{Fine-tuning}
\label{ftsec}
Our analysis relies critically on the linearization of the network under consideration. While such linearization results are well established in the context of over-parameterized networks \citep{linearnn}, our setting differs in that we allow non-random initialization and focus on architectures with transformer blocks. 

As a useful intermediate results, 
we  prove the network linearization rigorously in the pretrained initialization, which is crucial for helping derive the final theoretical bounds. This result can also be applied to interpret the success of task 
arithmetic (Section  \ref{tasksec}). 
\begin{lemma}[Network linearization]
\label{linlemma}
Under the setup in this paper, we are able to linearize the fine-tuned network around the pretrained model $f_{\text{PT}}$, i.e., we have
\begin{equation}
\begin{aligned}
      f_{\text{FT}} &=   f_{\text{PT}} + \Delta \theta \nabla f_{\text{PT}}(\theta) + e \\
      & = f_{\text{PT}}  + \sum_i  \zeta_i \mathbb{K} (\cdot,\mathbf{X}_i ) + e
\end{aligned}
\end{equation}
   where  the $L_2$-norm of the error term $e$ is bounded by
   $\tilde{\mathcal{O}}\left( \frac{1}{{d^l}^{1/2}}\right)$.
\end{lemma}

\begin{theorem}
\label{ftthm}
    Consider a network described in Section \ref{nnsetup} and has identical architecture as $f_{\text{PT}}$ , under assumptions  in Section 
\ref{assumptions}, 
for any $\gamma > 0$, 
there exists  sufficient large width and embedding dimension $M$, such that 
the following bounds hold 
with probability at least  $1-\gamma - c_1 \exp(-c_2 N_2 \widehat{\varrho}_{N_2})$
when fine-tuning the pretrained model $f_{\text{PT}}$ on the target dataset $X^{\text{tgt}}$ using gradient descent with learning rate $\epsilon = O(   \frac{1}{d^l}  )  $:
\label{cor1}
\begin{equation}
\begin{aligned}
      \|f_{\widehat{T}_{op}} - f_{\text{tgt}}\|_2^2       &\leq \mathcal{O}\left( N_2^{-\frac{1}{2}}\right).
\end{aligned}
\end{equation}
where $c_1$ and $c_2$ are some universal positive constants.
\end{theorem}

\begin{proofsketch}
The proof of this result relies heavily on the proof strategy in \cite{sunarxiv} for fully connected networks adapted to the attention-based network. Compared with \cite{sunarxiv}, the major difference of our paper is that we consider transformer decoder, which presents significantly more theoretical challenges due to the complexity of the Jacobian matrix discussed earlier.  

One of the key intermediate results we need to prove is the local Lipschitzness  of the Jacobian for the network under random normal initialization, which will later be extended to pre-trained initialization. The challenges come mainly from dealing with attention weights. After adding this complicated module, when considering the Jacobian of other fully connected weights, we need to be careful about the effects brought about by attention weights. In particular for attention weights,  we need to show that 
\begin{equation}
\begin{aligned}
 &\left\| J(W^{lh,\bullet})   \right\|_F \leq O({d^{l}}^{1/2}) \\
& \left\| J(W^{lh,\bullet})(\theta,x) - J(W^{lh,\bullet})(\tilde{\theta},x) \right\|_F \leq 
\tilde{\mathcal{O}}
({d^{l}}^{1/2}) \| \theta - \tilde{\theta}\|_2 
 \end{aligned}
\end{equation}
where $\bullet \in \{Q,K,V,O\}$.

Recall the expression for the Jacobian of attention weights in \ref{jatt},
to get the final bounds, we need to prove the boundness and local lipschitzness of all the elements inside those expressions. And this replies on the induction results from previous and later layers.  After this major step, we need the 
following steps to derive the final bound:
\begin{enumerate}
    \item Based on the local Lipschitzness of the Jacobian, we prove the NTK stability, network linearlization  for pre-trained network $f_{\text{PT}}$ under random random initialization
    and show that the parameter of $f_{\text{PT}}$ is not far away from the random initialized network $f_{\text{random\_init}}$.
    Specifically,
    let 
    $\theta_0$ denote the parameter of $f_{\text{random\_init}}$ and $\theta_{\text{PT}}$ denote that of $f_{\text{PT}}$. We formally show that $\theta_{\text{PT}}$ lies in the neighborhood of 
 $\theta_0$ with controlled radius, which is $B( \theta_0, C_1 n^{-1/2} )$, where $n$ is the width or hidden dimension of the network.
 .

    \item Based on previous step 1, we prove the local Lipschitzness of the Jacobian under the pre-trained initialization. The core idea of to prove this is to use the proved result is step 1, i.e.m the parameter of pre-trained network $\theta_{PT}$ is within the 
$B(\theta_{0},  C_1 {d^{l}}^{-\frac{1}{2}} )$ neighbor of some random weights $\theta_{0}$, then for any $\theta, \tilde{\theta} \in B(\theta_{\text{PT}},  C_2 {d^{l}}^{-\frac{1}{2}} )$, they also belong to a neighbor of  $\theta_{0}$, i.e. , $B(\theta_{0}, C_3 {d^{l}}^{-\frac{1}{2}}  )$ with a larger constant $C_3$. then by the local Lipschitzness of the random weights, we can get the local Lipschitzness  under pre-trained initialization. This process is  visualized  Figure \ref{proofsketchpic}.
\begin{figure}[htp!]

     \centering
     \begin{subfigure}[b]{0.4\textwidth}
    
         \centering
         \includegraphics[width=1\textwidth]{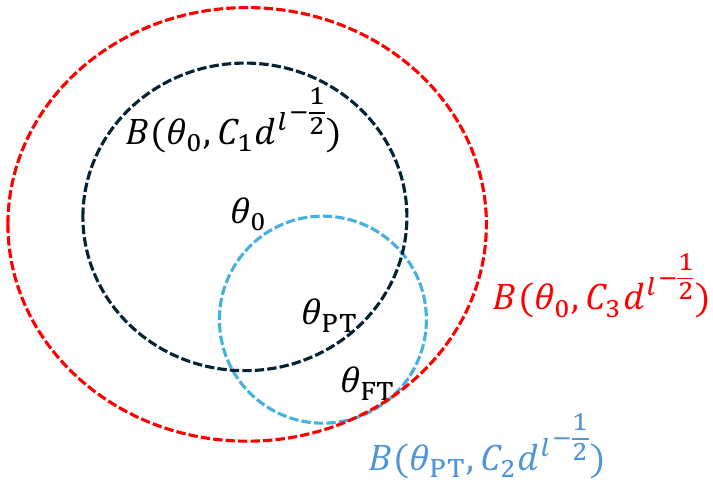}     
     \end{subfigure}
\caption{Proof of Local Lipschitzness around $f_{\text{PT}}$  illustration.
}
 \label{proofsketchpic}
\end{figure}

\item Based on Step 2, we establish NTK stability and network linearization for the pre-trained initialization using a proof strategy similar to that in \cite{sunaistats,sunarxiv}. More formally, we show that the empirical kernel remains nearly constant throughout training; that is,
\begin{equation}
\sup_{\tau}\left\|K_{\text{PT}}-K_{\tau}\right\|_F \leq \mathcal{O}({d^l}^{-\frac{1}{2}})
\end{equation}
Moreover, the fine-tuned network $f_{\text{FT}}$ can be linearized around the pre-trained network $f_{\text{PT}}$ as
\begin{equation}
f_{\text{FT}} \approx f_{\text{PT}} + \Delta \theta \nabla f_{\text{PT}}(\theta).
\end{equation}

\item Based on Step 3, we establish a convergence bound for the optimal stopping rule $\widehat{T}_{\text{op}}$ under gradient descent by leveraging proof techniques from \cite{esnonpara,sunaistats,sunarxiv}, which are specifically developed for the rigorous theoretical analysis of early stopping. Roughly speaking, the proof begins by decomposing the empirical normal of the mean prediction error as
    \begin{equation}
     \left\|f_\tau-f_{\text{tgt}}\right\|_{N_2}^2 \leq B_\tau^2+V_\tau+D_\tau^2
\end{equation}
where $B_\tau^2$ is the bias term, $V_\tau$ is the variance term and 
$D_\tau^2$ is the additional discrepancy term induced by the kernel variation during network training. We then bound each of these three terms at the desired rate, namely $\mathcal{O}({d^l}^{-\frac{1}{2}})$. Finally, we control the difference between the population $2$-norm $\|\cdot\|_2$ and the empirical norm $\|\cdot\|_N$ to obtain the final convergence bound stated in Theorem~\ref{ftthm}.

\end{enumerate}
The technical details of each step are deferred to the supplementary materials.

\end{proofsketch}

\begin{remark}
In the original NTK paper \citep{jacot}, the NTK is analyzed under the assumption of infinite width. In the infinite-width (or infinite hidden-dimension) regime, the change of the NTK during training converges to $0$, and consequently the difference term $D_\tau^2$ in the error bound vanishes. In contrast, our setup does not require such a strong assumption. Instead, we bound the NTK variation during training using concentration inequalities with respect to $d^l$. And we only require $d^l$ to be sufficiently large to ensure that $D_\tau^2$ is bounded by $\mathcal{O}( {d^l}^{-\frac{1}{2}} )$.
\end{remark}


\section{Specific fine-tuning settings}

In this section, we focus on two particular settings where fine-tuning is used in LLMs: linear probing and multitask learning with task arithmetic. 

\subsection{General Fine-tuning and Linear Probing}
\label{addheadsec}
We now turn to a more general fine-tuning setup. Specifically, suppose we retain only the last $k$ layers of the pretrained model as a backbone, denoted with $f_{\text{PT}}^k$ and append a linear head tailored to a new downstream task. The linear head is initialized with random Gaussian weights. During fine-tuning, we update the parameters of both the backbone and the newly added linear head \citep{ftref1,ftref2}. 


 This setting is closer to standard practice, since in most cases the output layer of the pretrained model does not align with that of the target task. The main theoretical challenge, however, lies in the fact that the initialization now combines non-random initialization (from the pre-trained backbone) with random initialization (from the newly added head), a regime for which existing results are not directly applicable.
 Since we only use the first $k$ layer of $f_{\text{PT}}$, the quantity $C_{\mathcal{H}}$ becomes
 \begin{equation}
     \left\| \text{LinearHead}  (f^k_{\text{PT}}) - f_{\text{tgt}}   \right\|_{\mathcal{H}}
 \end{equation}

Then we can achieve similar theoretical result as follows:
\begin{theorem}
\label{lhthm}
Consider a fine-tuning setting described above,
    under assumptions  in Section 
\ref{assumptions}, 
for any $\gamma > 0$, 
there exists  sufficient large width and embedding dimension $d^l$, such that 
the following bounds hold 
with probability at least  $1-\gamma - c_1 \exp(-c_2 N_2 \widehat{\varrho}_{N_2})$,
when fine-tuning \textbf{starting from 
$f_{\text{PT}}^k$ and add an additional linear head}
using gradient descent with learning rate 
$\epsilon = O(   \frac{1}{d^l} )$:
\label{cor1}
\begin{equation}
\begin{aligned}
      \|f_{\widehat{T}_{op}} - f_{\text{tgt}}\|_2^2       &\leq \mathcal{O}\left( N_2^{-\frac{1}{2}}\right).
\end{aligned}
\end{equation}
\end{theorem}

The overall proof framework is the same as Theorem \ref{ftthm}, but we need to adapt to the mixed initialization.

\begin{remark}
\textbf{Linear probing} freezes the backbone and trains only the head, reducing the problem to passing inputs through the backbone and solving a standard linear regression. In this case, the classical framework of \cite{esnonpara} applies directly, yielding analogous convergence results.
\end{remark}

\subsection{Task Arithmetic}
\label{tasksec}
For a particular task, the task vector $\tau$ is defined as the 
difference between $f_{\text{PT}}$ and $f_{\text{FT}}$ \citep{taskvector}. Specifically,
\begin{equation}
    \tau = \theta_{\text{FT}} - \theta_{\text{PT}}.
\end{equation}
In this section, we demonstrate how are fine-tuning approach yields similar results to prior work done on task arithmetic. We consider two types of task arithmetic: i) negation: users can negate a task vector to degrade performance on certain task, ii) addition: adding task vectors can help build better multi-task models.

\subsubsection{Negation}
To negate a task vector, we simply set the parameter of the negating network $f_{\text{neg}}$ as
\begin{equation}
    \theta_{\text{neg}} = \theta_{\text{FT}} - \tau.
\end{equation}
Applying Lemma \ref{linlemma} and 
Theorem \ref{ftthm}, it is straightforward to obtain the 
following result. 

\begin{corollary}
    For the network considered in this paper, under assumptions in Section 
    \ref{assumptions}
    ,
    for any $\gamma > 0$, there exist 
sufficiently large width and embedding dimension, 
when starting from $f_{\text{PT}}$ and using gradient descent with $\epsilon = O(   \frac{1}{d^l} )$, 
    the expected
    MSE of $f_{\text{neg}}$ is larger than that of 
    $f_{\text{FT}}$ 
    with probability at least
 $1-\gamma - c_5 \exp(-c_6 N_2 \widehat{\varrho}_{N_2})$
    , i.e.
    \begin{equation}
        \mathbb{E}[ MSE(f_{\text{neg}})] -  \mathbb{E}[ MSE(f_{\text{FT}})]  > 0.
    \end{equation}
\end{corollary}


\subsubsection{Addition}
\label{addsec}
Suppose we want to use the task vector to build a multi-task model $f_{\text{MultiTask}}$. For simplicity we consider two different tasks. Let the data for these two tasks as $X^{\text{task}_1}$ and $X^{\text{task}_2}$.
We add the below assumption:
\begin{assumption}
\label{orthass}
    the Hilbert space of $\text{span} \{ \mathbb{K}(\cdot, X^{\text{task}_1}) \}$ and  $\text{span} \{ \mathbb{K}(\cdot, X^{\text{task}_2}) \}$ are orthogonal.
\end{assumption}
This assumption seems quite strong, but task vectors are almost orthogonal according to the experiments in 
\cite{taskvector}. Note that this is a sufficient condition for task addition to work, if this do not hold, we cannot argue that task addition will lead to poor performance.

Denote the task vector for task 1 and 2 as $\tau_1$ and $\tau_2$, respectively. We have the following result.
\begin{corollary}
\label{orthcor}
    For the network considered in this paper, under assumption
    \ref{orthass}, and assumptions 
    in Section \ref{assumptions}, we are able to recover the disentanglement condition in \cite{tangenttask}. Specifically, 
    for any $\gamma > 0$, there exists sufficient large width and em-
bedding dimension $d^l$,  such that 
    the following holds with probability at least $1-r$ staring from $f_{\text{PT}}$ with learning rate $\epsilon = O(   \frac{1}{d^l} )$:
    \begin{equation}
    \begin{aligned}
         f_{\text{MultiTask}}  & =  f\left(X,\theta_{\text{PT} }
        + \sum_{i=1}^{2}\tau_i  
       \right) \\
       & =  \sum_{i=2}^2 g_i(X,\tau_i) + f_{\text{PT}} + e
    \end{aligned}
    \end{equation} 
where
$g_i(\mathbf{X}, \tau_i) = 0$ for $\mathbf{X} \notin X^{\text{task}_i}$. And the
$L_2$ norm of the error term $e$ is bounded by $\tilde{\mathcal{O}}\left( \frac{1}{{d^l}^{1/2}}\right)$.
\label{taskcor}
\end{corollary}

\section{Practical Insights} \label{practicalinsight}

Although the optimal stopping time $\widehat{T}_{op}$ is not directly computable in practice, the theoretical results provide a principled abstraction that yields concrete and actionable insights for real-world LLM fine-tuning after a closer examination.

\subsection{The theoretical framework 
encourages stopping early}
\label{taskdecaydis}

As for the question \textit{``why we only need a few epochs to get a good performance model?"}.
 We are able to answer it  after reviewing more details of the theoretical optimal stopping time $\widehat{T}_{\text{op}}$. As mentioned in Section \ref{stoprulesec}, we first define a threshold $\widehat{T}_{\max}$. Then 
$\widehat{T}_{\text{op}}$ is obtained by maximizing 
\begin{equation}
    \frac{C}{\epsilon \tau} + g(\tau)
\end{equation}
where $g(\tau)$ is an error term introduced by the accumulated kernel change during training and is non-decreasing  with training epoch $\tau$. The rationale for limiting training to only a few epochs emerges from two aspects of the theory. First, recall that  the constant $C_{\mathcal{H}}$ constrains $\widehat{T}_{\max}$ from being a large number, reflecting the expectation that $f_{\text{PT}}$ is already reasonably close to $f_{\text{tgt}}$. Second, the error term $g(\tau)$ favors fewer iterations, since running too long leads to greater accumulation of errors.

\subsection{Connection to kernel ridge regression
}
Early stopping is a classical method to prevent overfitting. And it has been shown that it is theoretically equivalent to kernel ridge regression in non-parametric regression settings \citep{esnonpara}. More specifically, the running sum of early stopping steps $\epsilon \widehat{T}$ acts the same as the ridge parameter $\lambda$ in the kernel ridge regression 
theoretically.  Later works  \citep{lazyvi,sunaistats,sunarxiv} extend this result to fully connected neural networks. We can apply similar reasoning in our case 
to derive the theoretical equivalence between the early stopping and kernel ridge regression
even though we consider way more complicated network structure. 

Formally, the early stopping can be roughly viewed as solving the below ridge regression problem:
\begin{equation}
\Delta \theta = \arg\min_{\omega }
    \frac{1}{N_2} \sum_{i=1}^{N_2} \left[  Y^{\text{tgt}}_i - f_{\text{PT}}( \mathbf{X}^{\text{tgt}}_i )  
    - w^T \nabla_{\theta} f_{\text{PT}}( \mathbf{X}^{\text{tgt}}_i)
    \right]^2 + \lambda \| w \|^2_2
\end{equation}
where $\lambda = \epsilon \widehat{T}_{\text{op}}$.
Then the fine-tuned model can viewed roughly as
\begin{equation}
    f_{\text{FT}} \approx  f_{\text{PT}} + \Delta \theta \nabla f_{\text{PT}}(\theta).
\end{equation}

In \cite{lazyvi}, the regularization parameter $\lambda$ is chosen via k-fold cross validation. The optimal stopping time 
$\widehat{T}_{\text{op}}$  can only be computed under ideal theoretical settings. So in practice, we recommend using the hold-out method, where the
training data is split into a validation set, and updates
are halted when the validation loss shows no improvement over several iterations.  Holdout early stopping can be loosely viewed as a degenerate case of cross-validation with a single validation split, but unlike k-fold cross validation, the validation data directly affects the training trajectory by determining the stopping time. Even though we adopt the hold-out method in our experiments, the optimal stopping rule proposed in the theoretical framework is  expected to outperform the hold-out method, which is supported by simulations (Figure 3) in \cite{esnonpara}.

\subsection{Task specific  convergence rate}

The convergence rate established in Theorem~\ref{ftthm} should be interpreted as an upper bound. In practice, the actual convergence behavior depends on the specific downstream task and the underlying network architecture \citep{sunarxiv, esnonpara}. In particular, it is closely tied to the eigenvalue decay rate of the population kernel matrix.

According to \cite{sunarxiv}, for neural networks one typically expects a polynomial eigenvalue decay, that is,
\begin{equation}
    \lambda_k \leq C(\frac{1}{k})^{2\beta}\ \  \quad 
    \text{for some } \beta > \frac{1}{2} \text{ constant } C.
    \label{polybeta}
\end{equation}
Under this condition, the resulting convergence rate of the mean prediction error is expected to satisfy
\begin{equation}
\begin{aligned}
      \|f_{\widehat{T}_{op}} - f_{\text{tgt}}\|_2^2       &\leq \mathcal{O}\left( N_2^{-\frac{2\beta}{2\beta + 1 }}\right)
\end{aligned}
\end{equation}
This implies that for sufficiently large $\beta$, the convergence rate approaches the parametric rate of $1/N_2$.

From a practical perspective, this result indicates that when comparing two downstream tasks with different eigenvalue decay rates, the task with faster decay can achieve comparable predictive performance using fewer training samples (up to constant factors). Consequently, the eigenvalue decay rate can serve as an informative measure of task–model compatibility in real-world fine-tuning scenarios. In particular,
 tasks exhibiting faster eigenvalue decay may require fewer labeled samples to reach a desired performance level for a given pre-trained model.  We further illustrate these insights empirically in the 
experiment (Section \ref{eigenexp}  ).


\section{EXPERIMENTS} \label{mainexp}




To support our theoretical results, we conduct experiments using open-source GPT models whose architectures align with our setting. In particular, we use  GPT-Neo-1.3B (16 heads, 128-dimensional embeddings), both pretrained on 825 GiB of text from 22 diverse sources \citep{black2021gptneo,black2022gptneox}. 

We consider several NLP tasks, which cover: 
semantic textual similarity (STS-B \citep{stsb}, STS-2012\citep{sts12}), semantic textual relatedness (SemRel \citep{semrel1,semrel2}
)
natural language inference (RTE), 
sentiment analysis (CR, SentEval-CR \citep{sencr}, Amazon Polarity\citep{amazon1}), paraphrase detection tasks (MRPC \citep{mrpc}). 
STS-B and STS-2012 are regression tasks and others are  classification tasks. So we need to add an additional classification head to the GPTNeo models, so this falls into the specific fine-tuned setting discussed in Section \ref{addheadsec}.

As noted in \cite{sunaistats}, the optimal stopping time $\widehat{T}_{\text{op}}$ is not computable in practice, so we adopt the same hold-out method to decide when to stop.
In our experiments, we first investigate 
relationship between
the eigenvalue decay  and the fine-tuned model performance, and then show that  when task vectors are nearly orthogonal across different tasks, task addition yields a multitask model $f_{\text{MultiTask}}$ with strong performance.

\subsection{Eigenvalue decay rate investigation}
\label{eigenexp}
We illustrate the connection between the eigenvalue decay rate and model performance using three  regression tasks:
SemRel, 
STS-B and STS 2012,  and three classification tasks: RTE, MRPC, and Amazon Reviews. Specifically, we randomly select 100 samples from each fine-tuning dataset and compute the empirical kernel matrix induced by the NTK of $f_{\text{PT}}$. We then fit a least-squares line to the log–log eigenvalue curve to estimate the polynomial decay coefficient $\beta$ in Equation (\ref{polybeta}).

For the regression tasks, the estimated mean decay rates are
$\beta = 0.60 \pm 0.0366$ for SemRel
,
$\beta = 0.58 \pm 0.0126$ for STS-12 and $0.54 \pm 0.0229$ for STS-B. For the classification tasks, the corresponding values are $\beta = 0.62 \pm 0.0076$, $0.55 \pm 0.0029$, and $0.51 \pm 0.01$ for Amazon Reviews, MRPC, and RTE, respectively. The eigenvalue decay curves are reported in the Appendix.
As discussed in Section~\ref{taskdecaydis}, this predicts that fewer training samples are required for SemRel than for STS-12, and fewer for STS-12 than for STS-B, consistent with the fact that STS-12 is a cleaner subset of the data sources aggregated into STS-B. Similar conclusions are expected to hold for the three classification tasks.

We verify this prediction by evaluating fine-tuning performance under varying training sample sizes.
For regression tasks, test performance is measured by MSE loss, while for classification tasks it is measured by test accuracy, as shown in Figure~\ref{figeigncompare}. The F1-score comparison for classification tasks is depicted in Figure~\ref{f1supp} in the Appendix.
For the regression tasks, we observe that SemRel performs best, followed by STS-2012, with STS-B performing the worst.  Across all tasks, Amazon Reviews converges the fastest, reaching near-optimal performance with fewer than 500 training samples, followed by MRPC, while RTE exhibits the slowest convergence and the lowest overall performance.
These results provide empirical support for the claims made in Section~\ref{taskdecaydis}.

\begin{figure}[htp!]
     \centering
     \begin{subfigure}[b]{0.4\textwidth}
         \centering
         \includegraphics[width=1\textwidth]{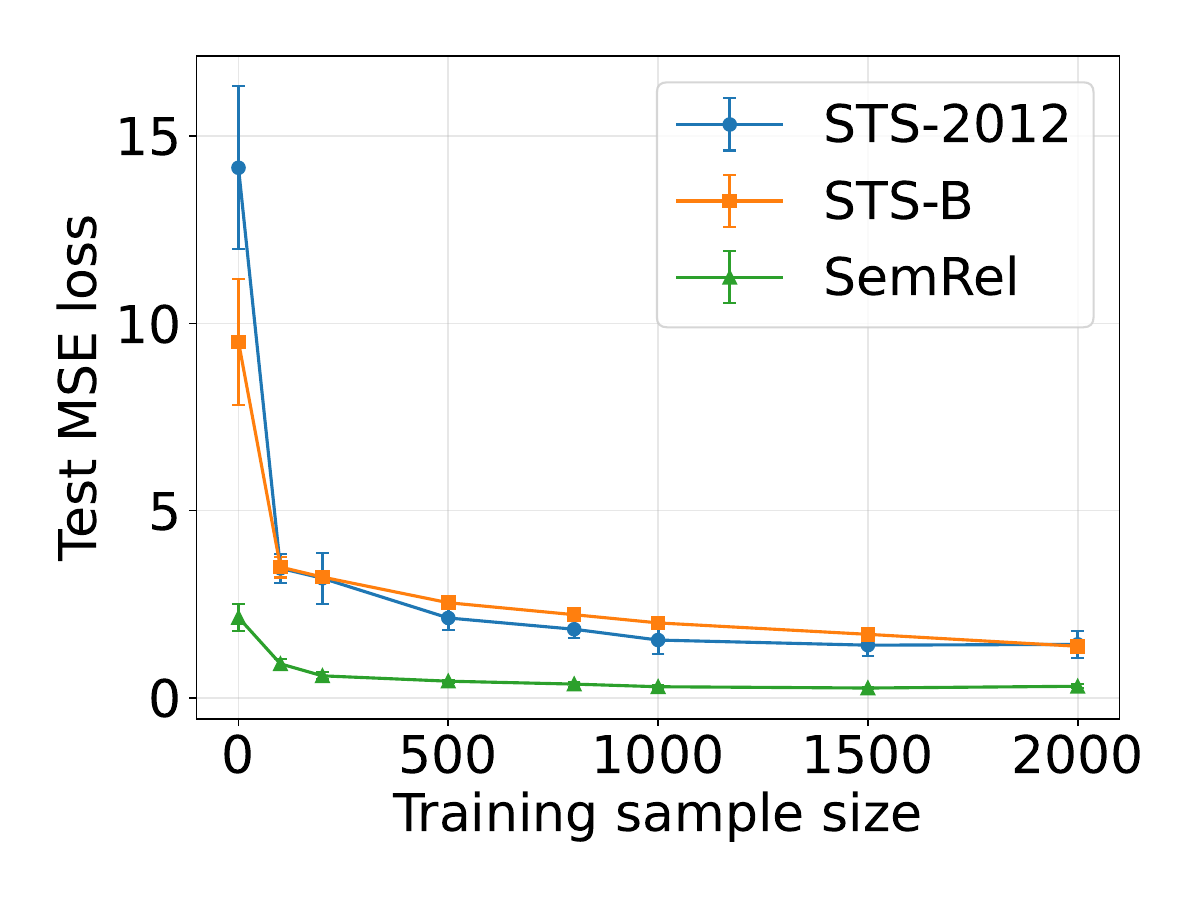}  
         
     \end{subfigure}
     \begin{subfigure}[b]{0.4\textwidth}
         \centering
         \includegraphics[width=1\textwidth]{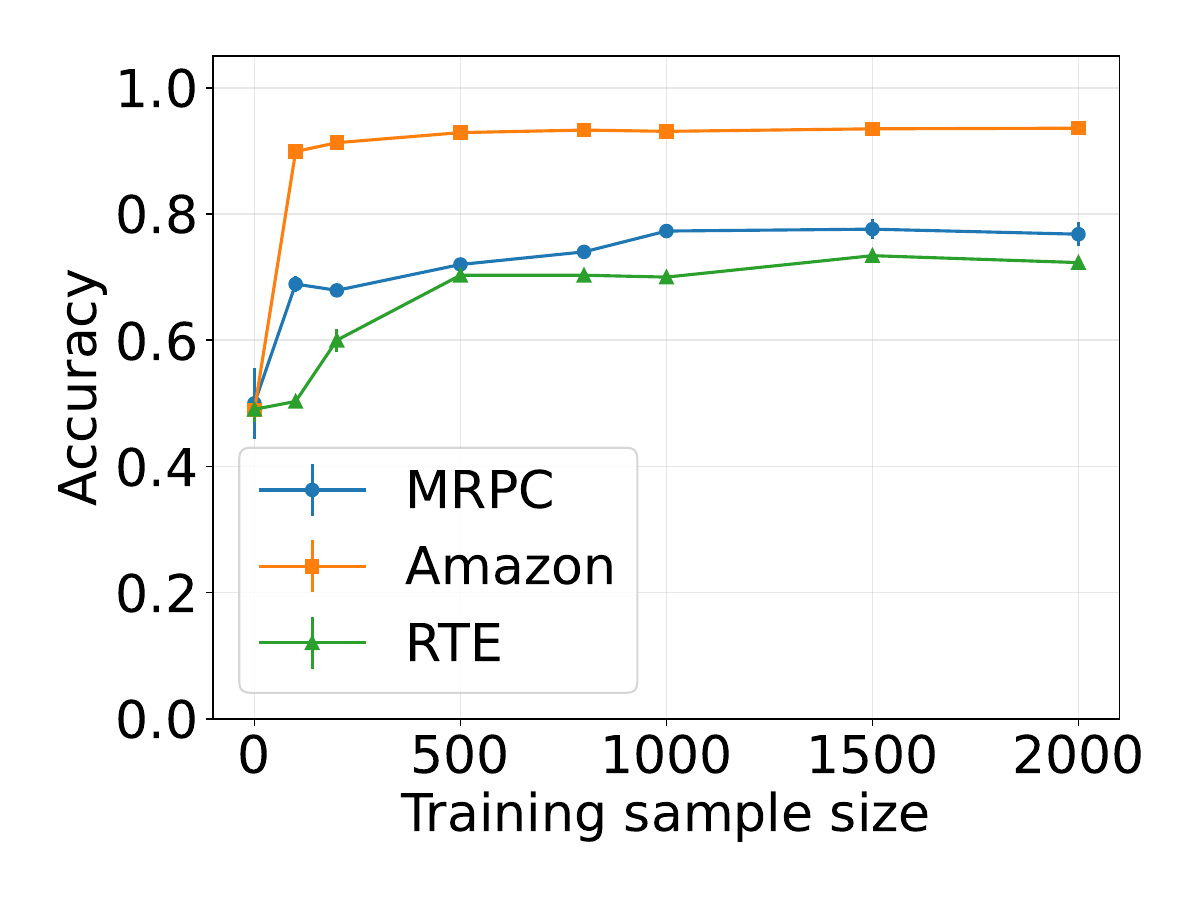}

     \end{subfigure}
\caption{  
Mean learning curves over three random runs: MSE loss for STS-12 and STS-B (left), accuracy for RTE, MRPC, and Amazon Reviews (right). 
RemSem has the largest decay rate ($\beta = 0.60$), followed by STS-12 ($\beta = 0.58$), and STS-B the smallest ($\beta = 0.54$), reflecting the performance ranking. For classification, $\beta=0.62, 0.55, 0.51$ for Amazon Reviews, MRPC, and RTE, respectively, matching their observed performance ranking.
}
\label{figeigncompare}
\end{figure}

\subsection{Task Vector}
\label{taskvecorthoexp}
We primarily use GPT-Neo 1.3B in this experiment, as its setup aligns more closely with our theoretical assumptions. For a task pair $(\text{task}_1, \text{task}_2)$, we first fine-tune the model separately on each task to obtain $f_{\text{FT}_1}$ and $f_{\text{FT}_2}$. From these, we derive the corresponding task vectors $\tau_1$ and $\tau_2$. The multi-task model is then constructed by adding $\tau_1$ and $\tau_2$ to the pre-trained parameters $f_{\text{PT}}$ via task vector addition. We compute the cosine similarity between $\tau_1$ and $\tau_2$ across four datasets  and compare the performance of $f_{\text{MultiTask}}$ against the individually fine-tuned models (Figure \ref{taskfig}).


 \begin{figure}[htp!]
     \centering
     \begin{subfigure}[b]{0.4\textwidth}
         \centering
         \includegraphics[width=1\textwidth]{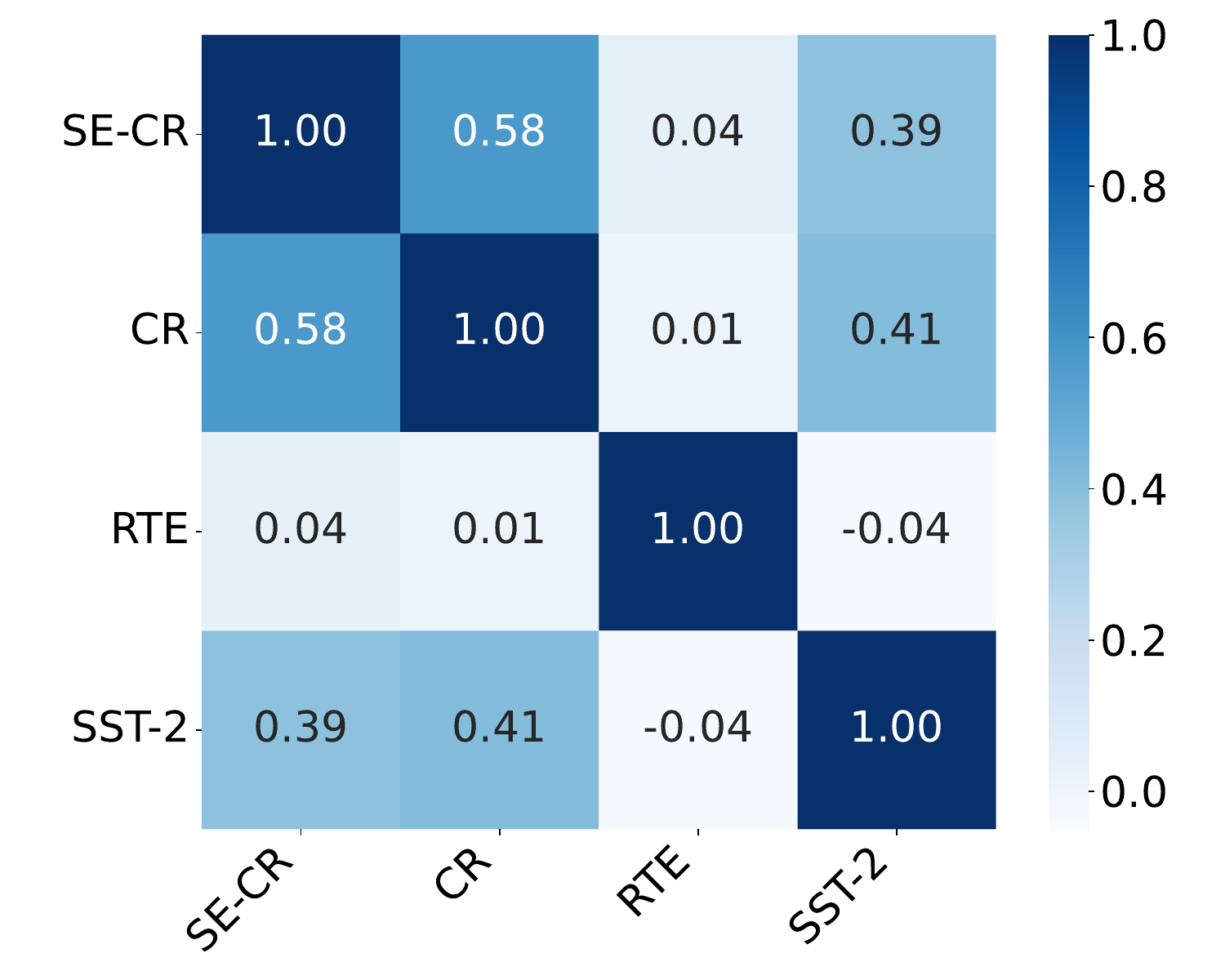}  
     \end{subfigure}
     \begin{subfigure}[b]{0.4\textwidth}
         \centering
         \includegraphics[width=1\textwidth]{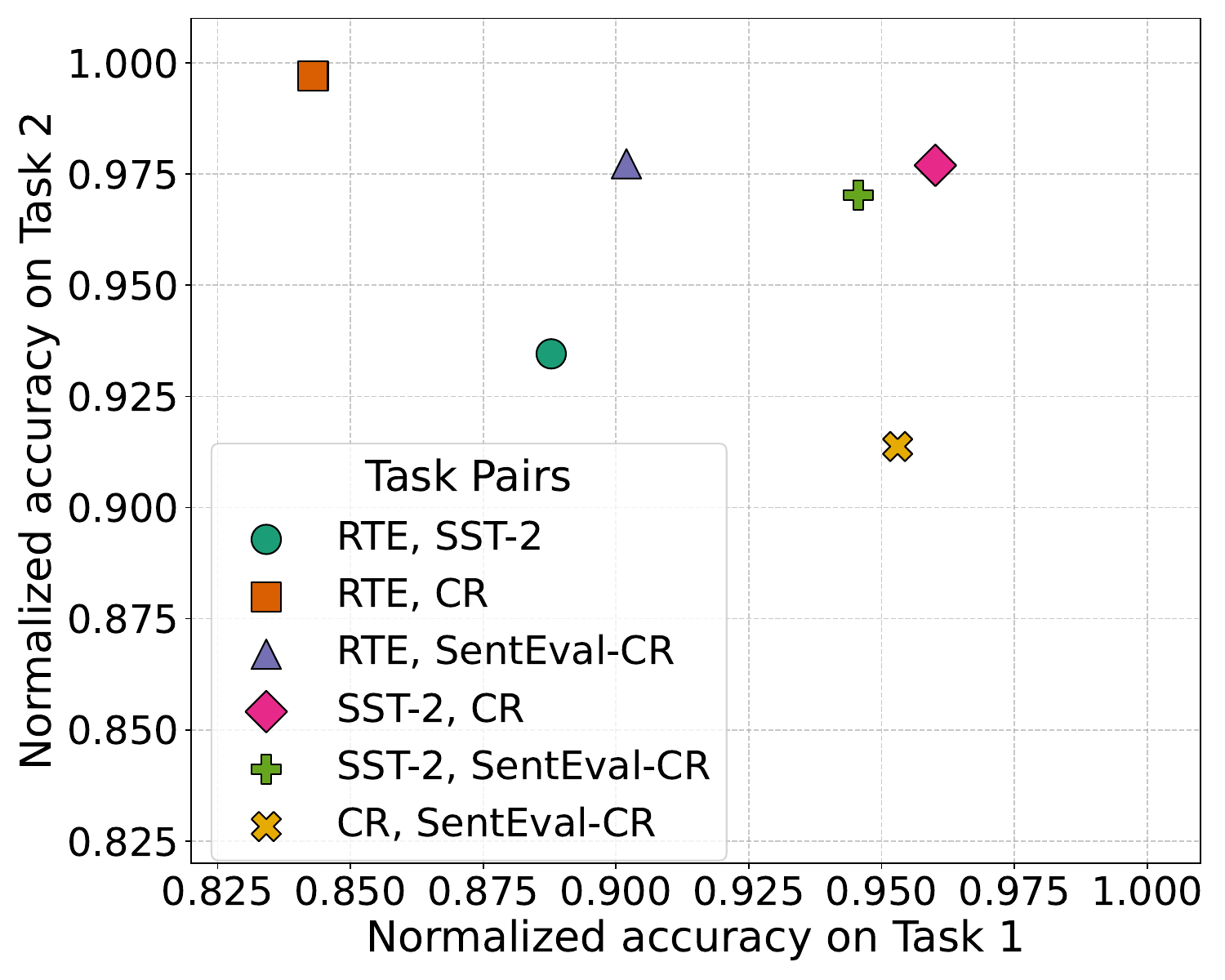}
     \end{subfigure}
\caption{
Task vector cosine similarity heatmap, SE-CR is shorthand of SentiEval-CR (left);  Normalized accuracy of $f_{\text{MultiTask}}$ compared with fine-tuned model (right). Task vector of RTE is almost orthogonal to other task vectors and task addition still works well even when task vectors are correlated. 
}
\label{taskfig}
\end{figure}
We observe that the task vector of RTE is nearly orthogonal to the other task vectors, and $f_{\text{MultiTask}}$  achieves strong performance on individual tasks, which is consistent with Corollary \ref{orthcor}. Another key finding is that even when task vectors are not orthogonal, task addition can still yield good performance for $f_{\text{MultiTask}}$. Importantly, we do not claim that correlation between task vectors or tasks necessarily leads to failure of task addition, so this observation does not contradict our theoretical results.

\section{DISCUSSION}
In this paper, we focus on transformer-based networks and establish a theoretical connection between modern fine-tuning practices and early stopping. Our analysis provides 
valuable practical insights about LLM finetuning  from a rigorous perspective. 
 The theoretical results also offer a more formal interpretation of the effectiveness of task arithmetic. While our work emphasizes a theoretical perspective, experiments with modern LLMs further support the analysis, highlighting the practical relevance of our framework.

The alignment between our theoretical predictions and the widespread success of pre-trained models in practice suggests that the proposed framework has the potential to extend to broader settings. In particular, it may accommodate alternative loss functions, additional components of modern deep learning architectures, and more advanced training techniques.

\bibliographystyle{unsrtnat}
\bibliography{references}

\appendix

\section{Experimental Details}
\subsection{Datasets}
\begin{table}[htp!]
\centering
\begin{tabular}{llllll}
\hline
\textbf{Dataset} & $C$ & \textbf{Train} & \textbf{Test} &  \textbf{Type} & \textbf{Label words} \\
\hline
MRPC  & 2 & 3,668   &   408  & 
paraphrase 
& \{Yes, No\} \\
RTE       & 2 & 2,490   & 277  &  NLI & \{Yes, No\} \\
CR    & 2 & 3,390   & 376 &  sentiment & \{positive, negative\} \\
SentEval-CR  & 2   & 3,010 & 753   & sentiment  &  \{positive, negative\} \\
Amazon Reviews & 2   & 3.6M &  400k   & sentiment  &  \{positive, negative \}\\
\hline
\end{tabular}
\caption{The statistics of the classification datasets we used in our experiments. }
\end{table}

\begin{table}[htp!]
\centering
\begin{tabular}{lllllc}
\hline
\textbf{Dataset}  & \textbf{Train} & \textbf{Test} &  \textbf{Type} & \textbf{Score range} \\
\hline
SST-12  & 2,234   & 3,108  &  similarity  & 0-5 \\
SST-B    & 5,749  & 1,379 & similarity  & 0-5 \\
\hline
\end{tabular}
\caption{The statistics of the regression datasets we used in our experiments. }
\end{table}

The CR  and SentEval-CR datasets are from \cite{sencr},
MRPC is from \citep{mrpc}, RTE is from \cite{rte}
Amazon reviews dataset is from \cite{amazon1}, SST-12 is from 
\cite{sts12}
and  STS-B is from  \cite{stsb}.

\subsection{Computing the Kernel}
We employ functorch \citep{functorch} to compute the NTK for GPT-Neo 125M and GPT-Neo 1.3B, combining backward-mode automatic differentiation to obtain Jacobians with forward-mode differentiation to evaluate Jacobian–vector products \citep{novak2022fast}.

\subsection{Additional experiment results}
The eigenvalue decay curves for different tasks are shown in Figure 
\ref{amazonreviews} -- \ref{semralbeig}.  The process to get the estimated $\beta$   is as follows. Suppose  the eigenvalues follow a power law:
\begin{equation}
    \lambda_k  \approx  \lambda_1 k^{-2\beta}
\end{equation}
taking logs:
\begin{equation}
    \log(\lambda_k / \lambda_1) \approx -2\beta \log k 
\end{equation}
we then run a least square estimate on the log ration $\log(\lambda_k / \lambda_1)$ and $\log k$ to estimate the slope $-2\beta$. 

The f1 score comparison for the three classification tasks are shown in Figure \ref{f1supp}. We can see that we observe similar pattern as the accuracy learning curve. 
\begin{figure}[htp!]
     \centering
     \begin{subfigure}[b]{0.3\textwidth}
         \centering
         \includegraphics[width=1\textwidth]{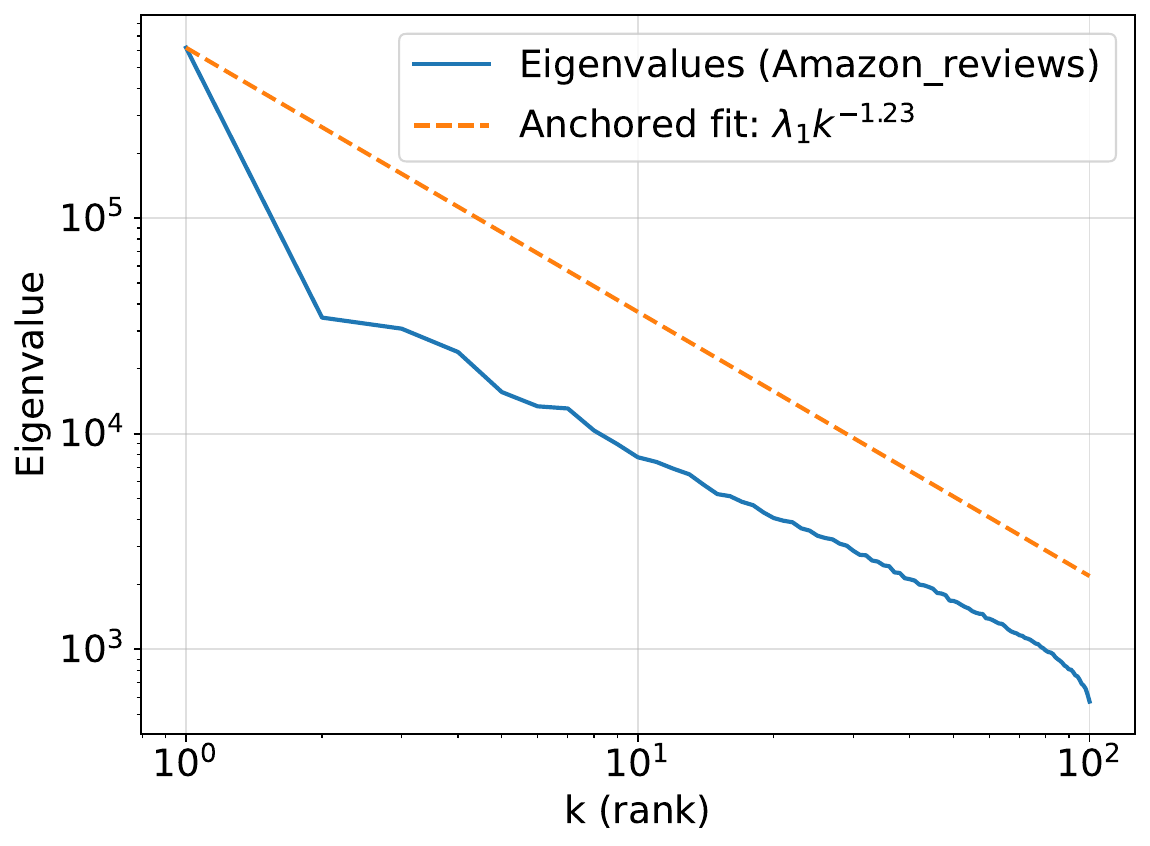}   
     \end{subfigure}
     \begin{subfigure}[b]{0.3\textwidth}
         \centering
         \includegraphics[width=1\textwidth]{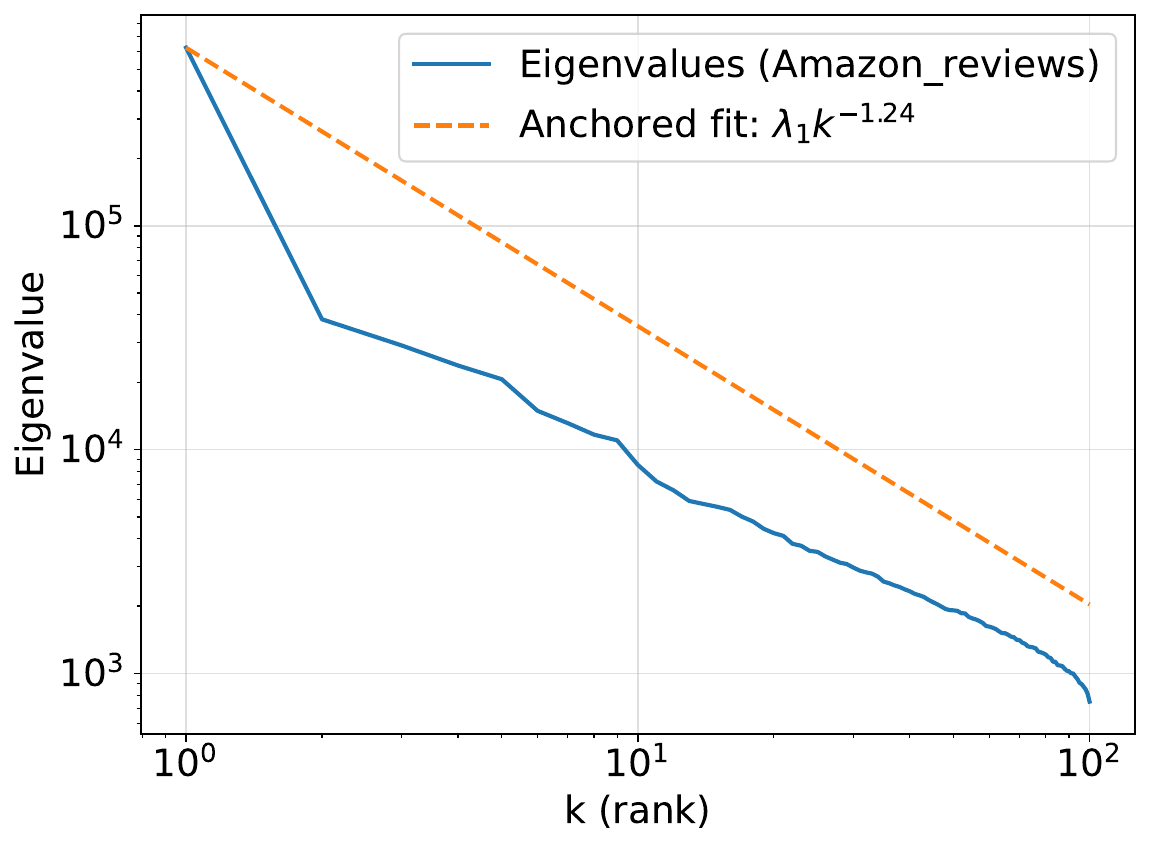}
        
     \end{subfigure}
     \begin{subfigure}[b]{0.3\textwidth}
         \centering
         \includegraphics[width=1\textwidth]{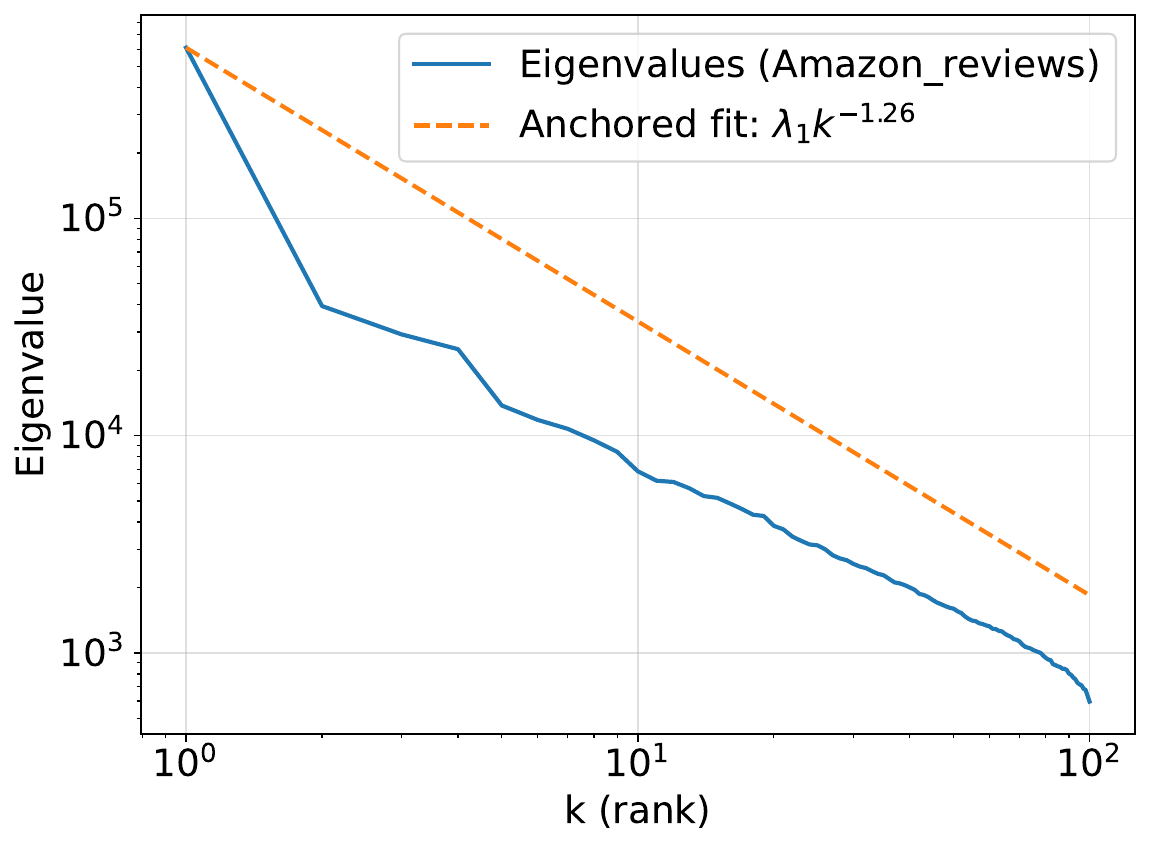}
     \end{subfigure}
\caption{Amazon reviews eigenvalues decay curve. 
}
\label{amazonreviews}
\end{figure}

\begin{figure}[htp!]
     \centering
     \begin{subfigure}[b]{0.3\textwidth}
         \centering
         \includegraphics[width=1\textwidth]{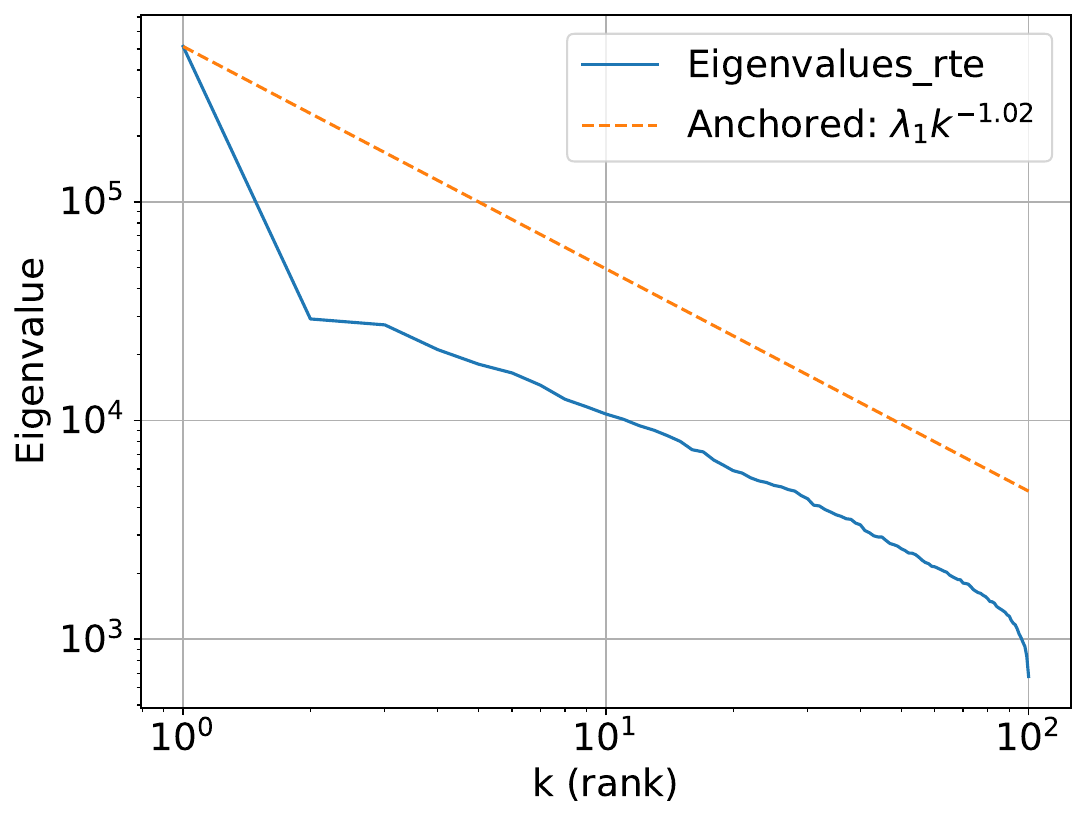}   
     \end{subfigure}
     \begin{subfigure}[b]{0.3\textwidth}
         \centering
         \includegraphics[width=1\textwidth]{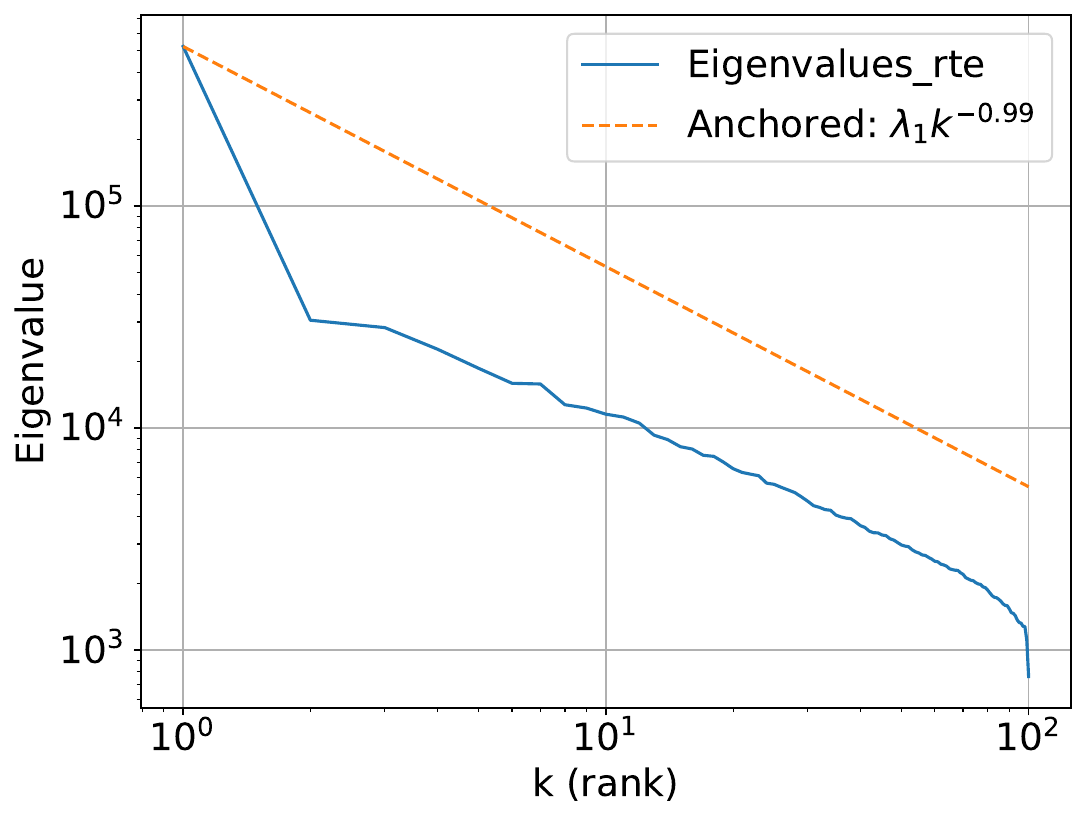}
        
     \end{subfigure}
     \begin{subfigure}[b]{0.3\textwidth}
         \centering
         \includegraphics[width=1\textwidth]{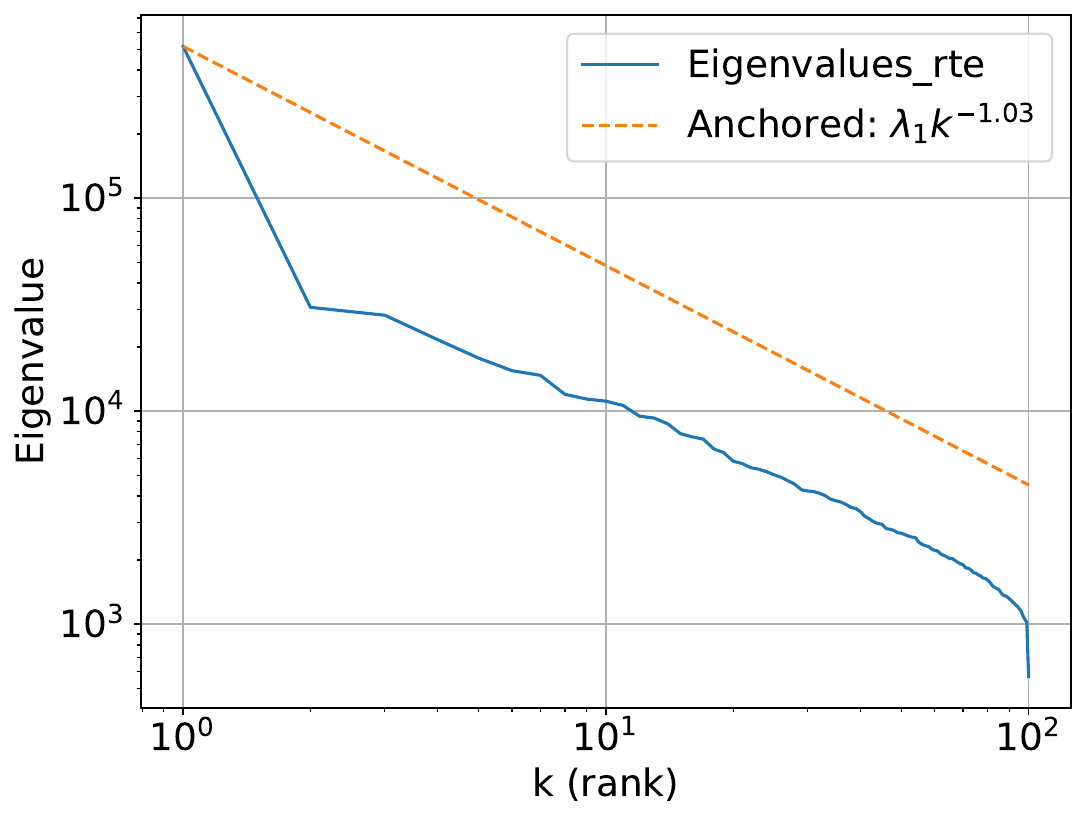}
     \end{subfigure}
\caption{RTE eigenvalues decay curve.
}
\end{figure}

\begin{figure}[htp!]
     \centering
     \begin{subfigure}[b]{0.3\textwidth}
         \centering
         \includegraphics[width=1\textwidth]{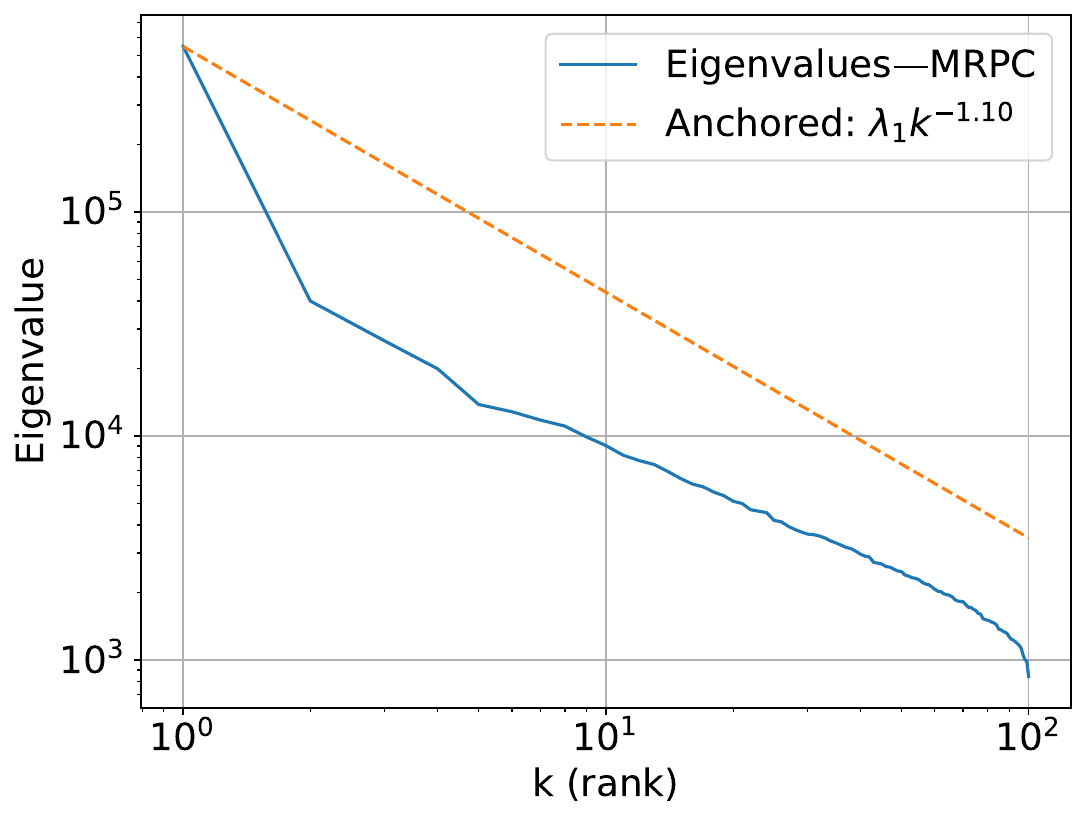}   
     \end{subfigure}
     \begin{subfigure}[b]{0.3\textwidth}
         \centering
         \includegraphics[width=1\textwidth]{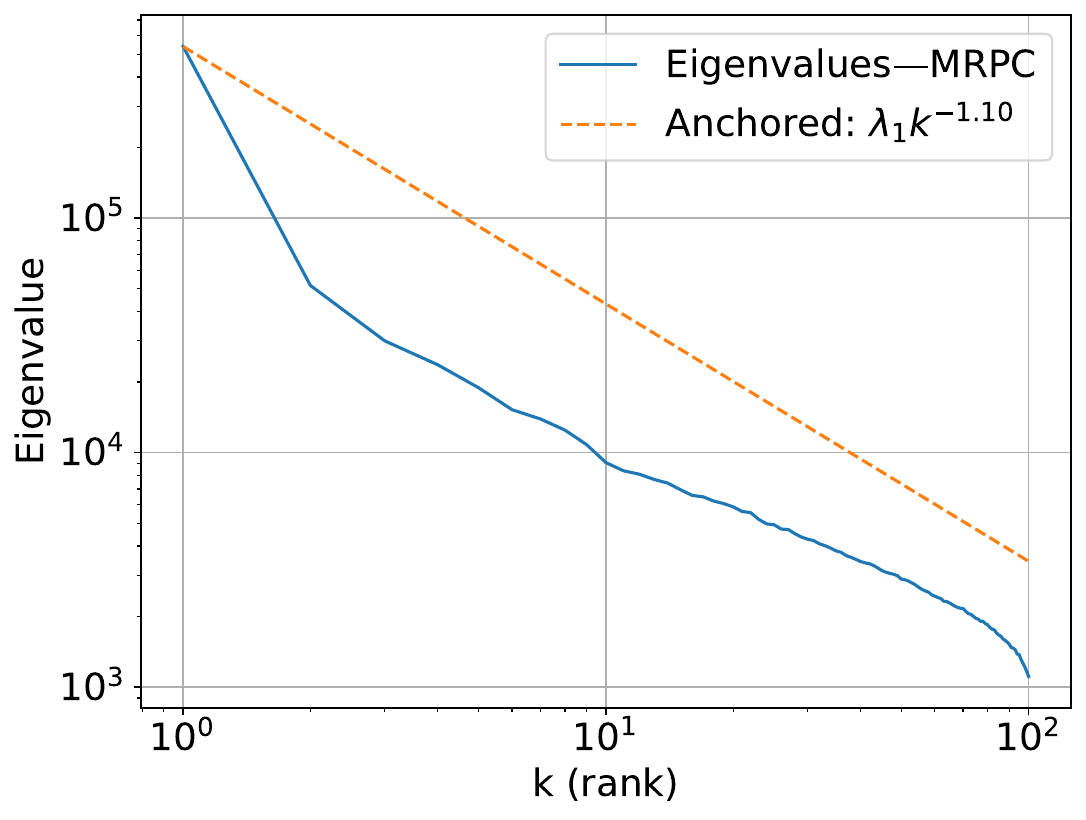}
        
     \end{subfigure}
     \begin{subfigure}[b]{0.3\textwidth}
         \centering
         \includegraphics[width=1\textwidth]{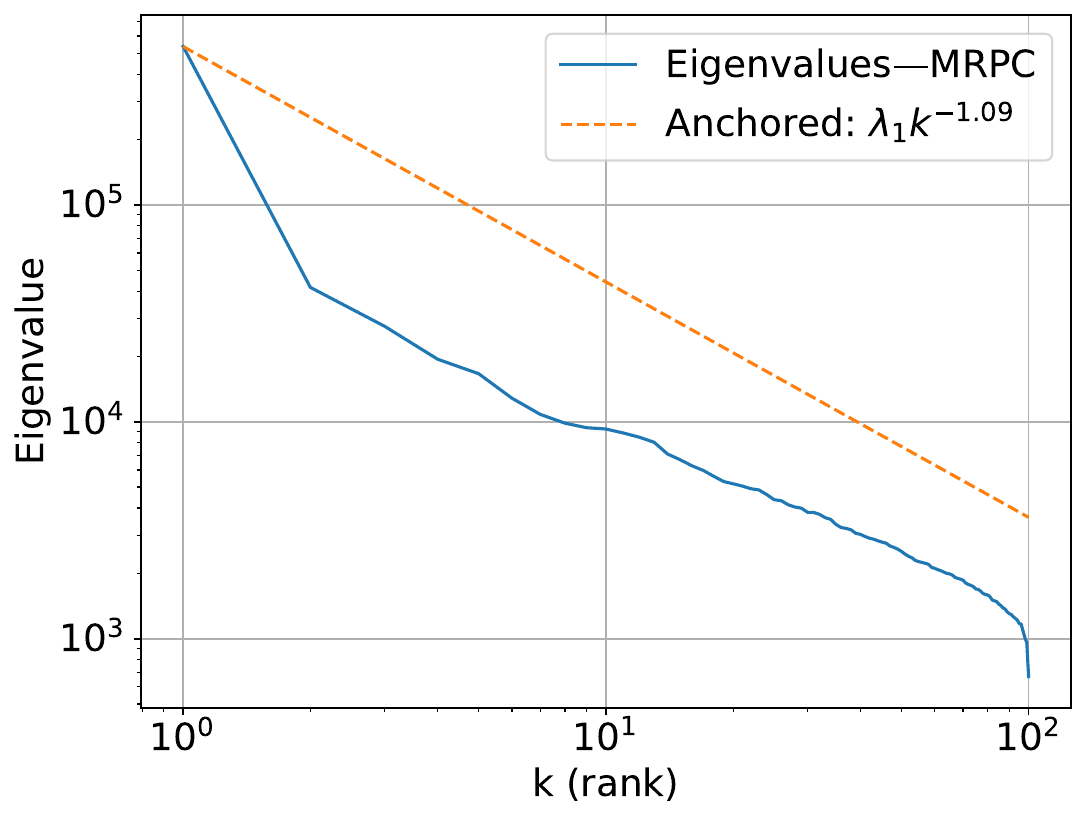}
     \end{subfigure}
\caption{MRPC eigenvalues decay curve.
}
\end{figure}

\begin{figure}[htp!]
     \centering
     \begin{subfigure}[b]{0.3\textwidth}
         \centering
         \includegraphics[width=1\textwidth]{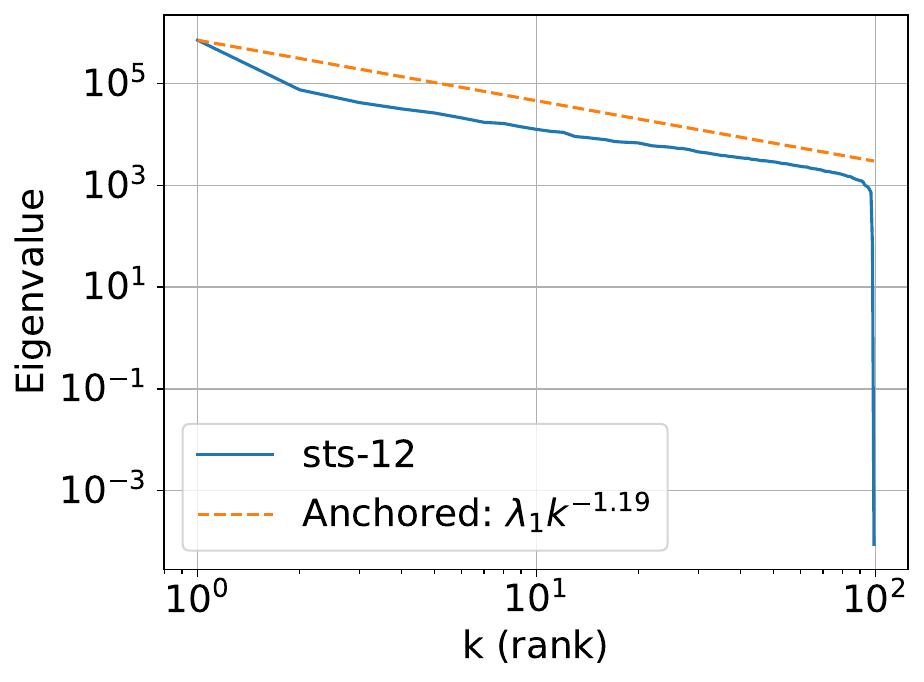}   
     \end{subfigure}
     \begin{subfigure}[b]{0.3\textwidth}
         \centering
         \includegraphics[width=1\textwidth]{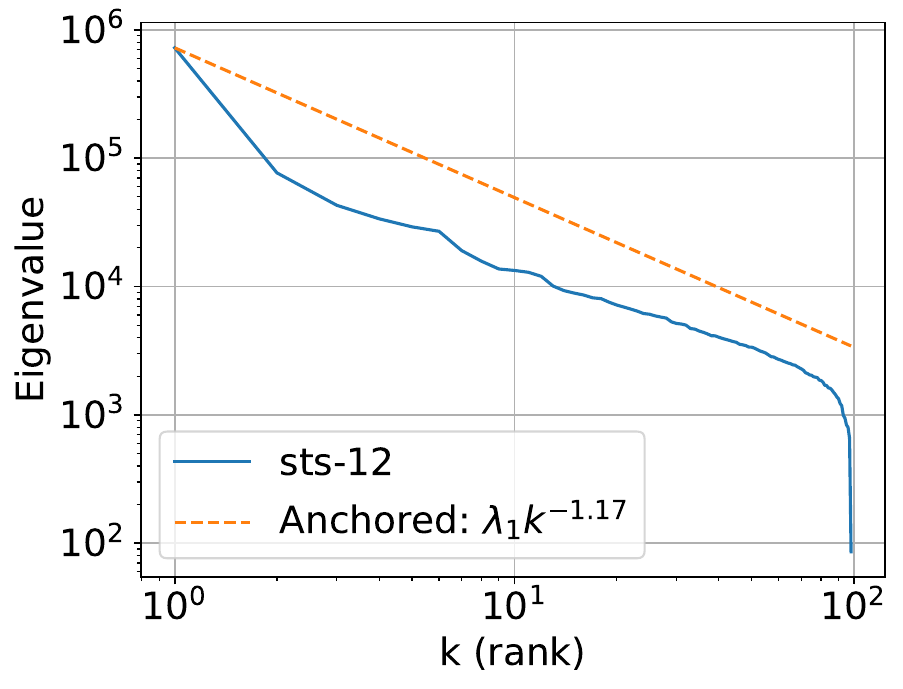}
        
     \end{subfigure}
     \begin{subfigure}[b]{0.3\textwidth}
         \centering
         \includegraphics[width=1\textwidth]{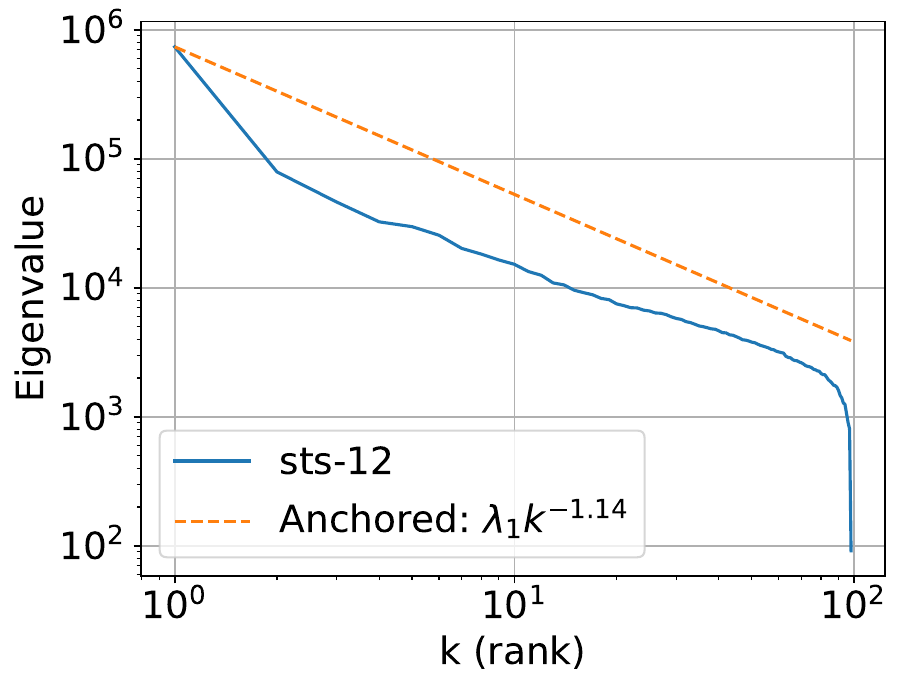}
     \end{subfigure}
\caption{STS-12 eigenvalues decay curve.
}
\end{figure}

\begin{figure}[htp!]
     \centering
     \begin{subfigure}[b]{0.3\textwidth}
         \centering
         \includegraphics[width=1\textwidth]{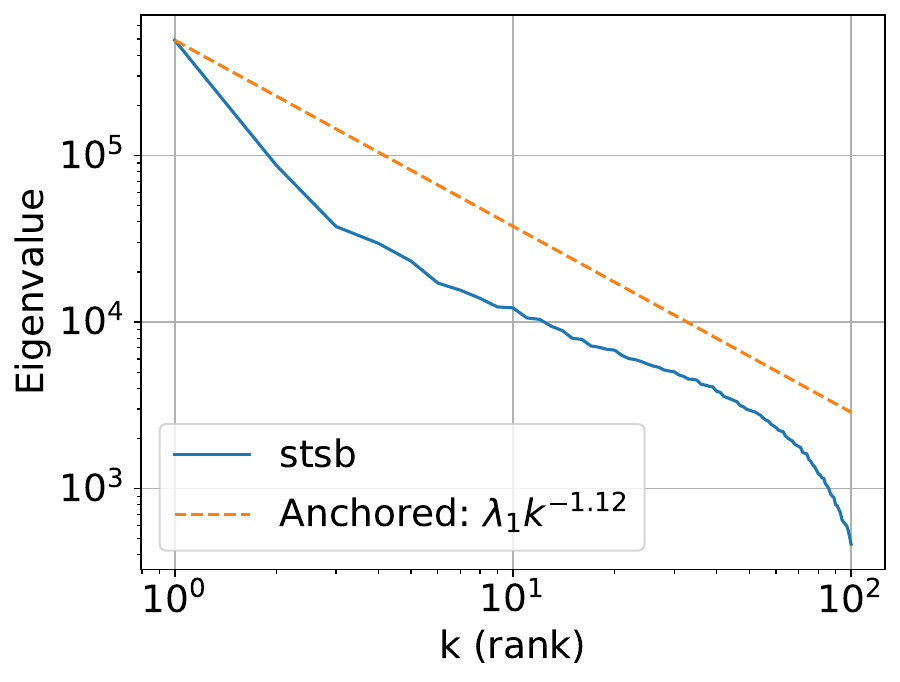}   
     \end{subfigure}
     \begin{subfigure}[b]{0.3\textwidth}
         \centering
         \includegraphics[width=1\textwidth]{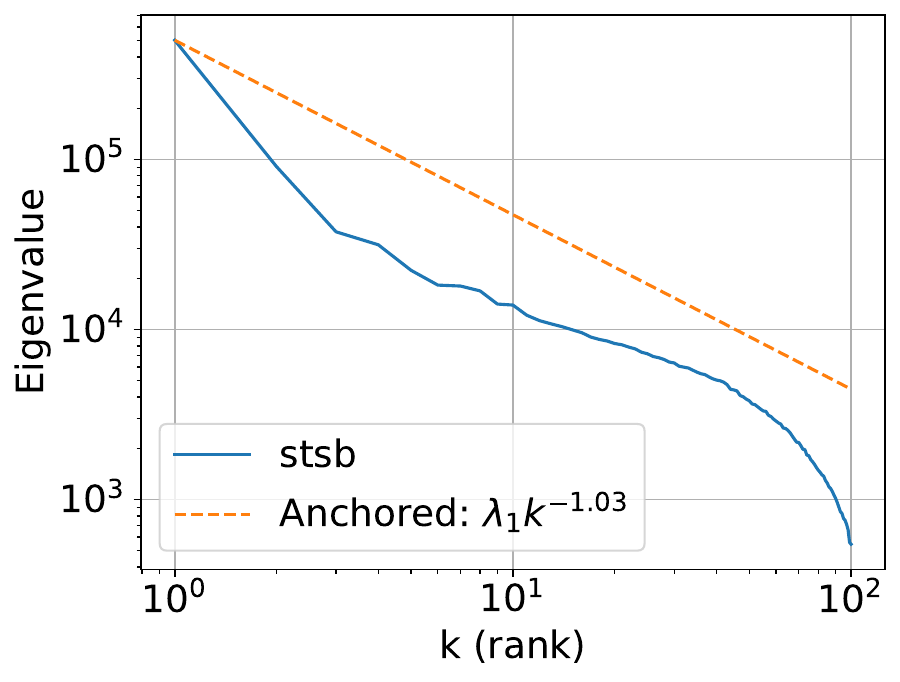}
        
     \end{subfigure}
     \begin{subfigure}[b]{0.3\textwidth}
         \centering
         \includegraphics[width=1\textwidth]{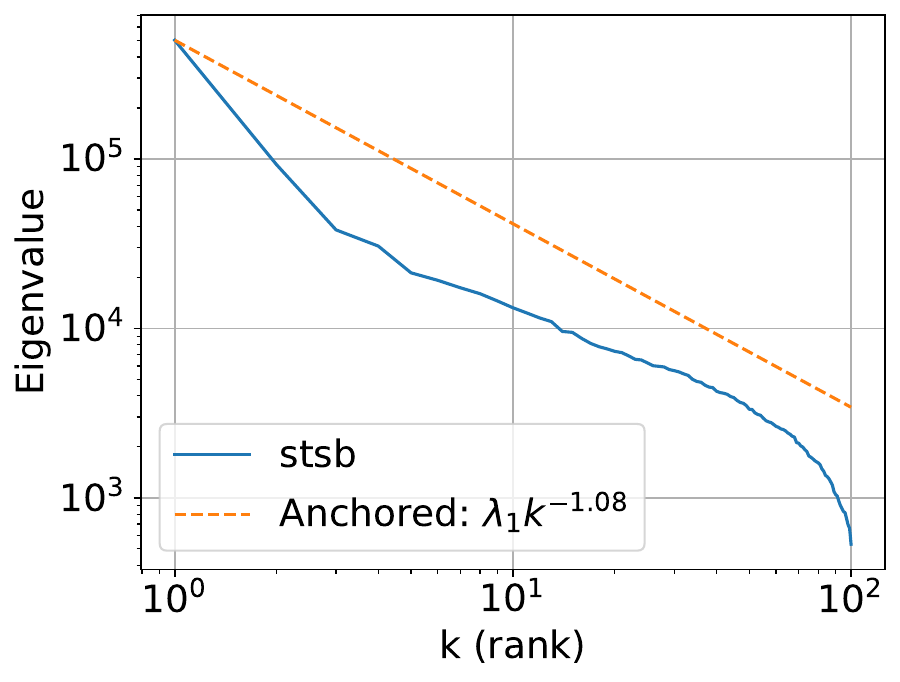}
     \end{subfigure}
\caption{STS-B eigenvalues decay curve.
}
\label{stsbeig}
\end{figure}

\begin{figure}[htp!]
     \centering
     \begin{subfigure}[b]{0.3\textwidth}
         \centering
         \includegraphics[width=1\textwidth]{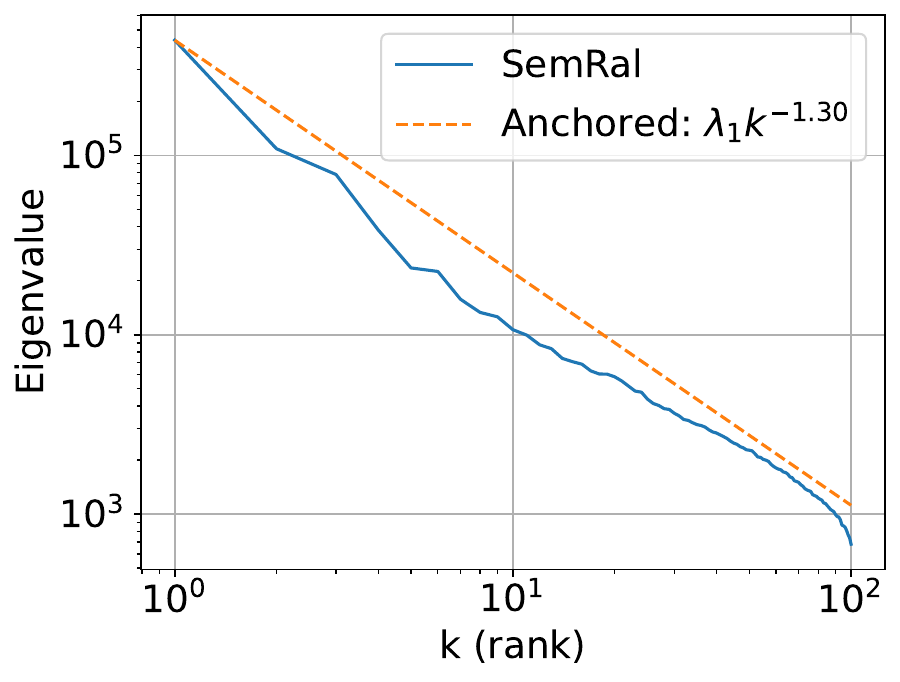}   
     \end{subfigure}
     \begin{subfigure}[b]{0.3\textwidth}
         \centering
         \includegraphics[width=1\textwidth]{figures/gptneo1.3B_sts_eigen_decay_1.pdf}
        
     \end{subfigure}
     \begin{subfigure}[b]{0.3\textwidth}
         \centering
         \includegraphics[width=1\textwidth]{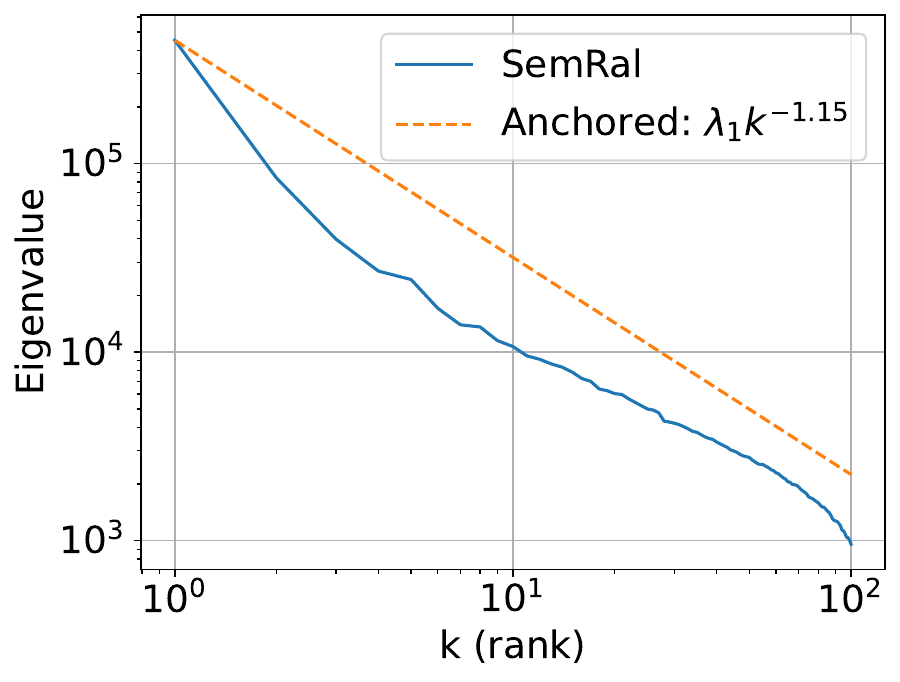}
     \end{subfigure}
\caption{SemRal eigenvalues decay curve.
}
\label{semralbeig}
\end{figure}


\begin{figure}[t!]
     \centering
     \begin{subfigure}[b]{0.45\textwidth}
         \centering
         \includegraphics[width=1\textwidth]{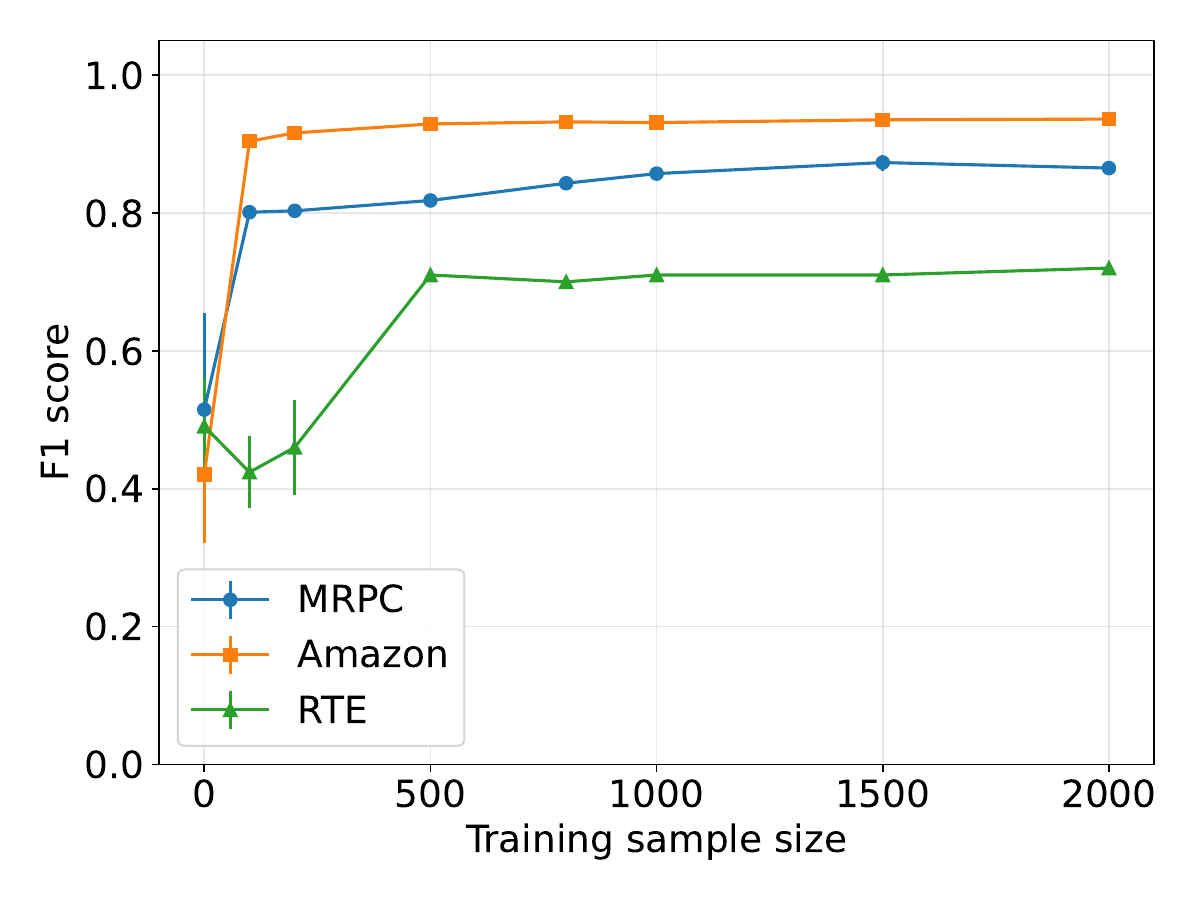}   
     \end{subfigure}
\caption{F1 score comparison of MRPC, RTE, and Amazon reviews datasets.}
\label{f1supp}
\end{figure}



\section{Theoretical Results}

\subsection{Optimal Stopping Rule Details}
The optimal stopping rule $\widehat{T}_{\text{op}}$ mentioned in Section
3.5 of the main text
is adopted directly from \cite{sunaistats}. Before discussing its details, we first define the 
 $L^2\left(P_N\right)$-norm as:
\begin{equation}
\left\|\hat{f}-f^*\right\|_N^2=\frac{1}{N} \sum_{i=1}^N\left(\hat{f}\left(\mathbf{X}_i\right)-f^*\left(\mathbf{X}_i\right)\right)^2
\end{equation}
The we can bound the  $L^2\left(P_N\right)$-norm
between $f_{\tau}$ (initialized with 
$f_{\text{PT}}$
) and $f_{\text{tgt}}$ as follows:
    \begin{equation}
        \left\|f_\tau-f_{\text{tgt}}\right\|_N^2 \leq 
      B_\tau^2 +  V_\tau + 
        D_\tau^2
        \label{erdecom}
    \end{equation}
where $B_\tau^2$ is the bias term, 
$V_\tau$
 is the variance term, and $D_\tau^2$
is the additional difference term.  The key difference between the early stopping framework in \cite{esnonpara} and that in \cite{sunaistats} is that the latter accounts for modern non-linear models with evolving kernels during training, which introduce the extra error component $D_\tau^2$.

According to \cite{esnonpara}, for any iteration $\tau \leq \widehat{T}_{\max }$,  with certain probability,  we can bound the sum of the bias and variance terms as
\begin{equation}
    B_\tau^2 + V_\tau \leq \frac{C}{\eta_{\tau}}. 
\end{equation}
Meanwhile, the difference term $D^2_{\tau}$ is non-decreasing in $\tau$. Suppose $D^2_{\tau}$ is upper bounded by a non-decreasing function $g(\tau)$.
Then the proposed optimal stopping time $\widehat{T}_{\text{op} }$ is defined as:
\begin{equation}
    \widehat{T}_{\text{op} }:=\arg \min \left\{\tau \leq \widehat{T}_{\max }  \left\lvert\, 
  \frac{C}{\eta_{\tau}} + 
        g(\tau)  \right.\right\}.
\end{equation}
For details of the proof of these intermediate
 results please refer to \cite{sunarxiv} and \cite{esnonpara}.

\subsection{Network specification Details}

For theoretical convenience, we set the embedding dimension for each layer $d^{l}$, and the scaling factor $d^{l, G}$
to be the same as the network width $n$. 
Same as \cite{sunaistats}, we consider $n$ to be sufficiently large, so we can bound the error 
introduced by the evolving kernel during training, i.e. , the difference term $D^2_{\tau}$. 
Note that we consider the number of heads $d^{l, H}$ to be finite, which is more realistic.

We use \textit{standard parameterization} in the main paper and give theoretical results accordingly.  But note that our results also hold for  \textit{NTK parameterization}, where  we implicitly treat each weight $W_{i j} \sim \mathcal{N}\left(0, \sigma^2 / n\right)$, i.i.d., as $W=\frac{\sigma}{\sqrt{n}} w $ where only $w$ is trainable.
For \textit{standard parameterization}, we should normalize the NTK  with the width or embedding dimension $n$:
\begin{equation}
   \frac{1}{n} \left\langle
     \nabla_{\theta} f\left(\theta(t), x\right),
     \nabla_{\theta} f\left(\theta(t), x^{\prime}\right)\right\rangle.
\end{equation}

And for the empirical kernel matrix $K$, we denote  $0<\lambda_{\min }:=\lambda_{\min }(K) \leq$ $\lambda_{\max }:=\lambda_{\max }(K)<\infty$. And let $\eta_{\text {critical }}=2\left(\lambda_{\min }+\lambda_{\max }\right)^{-1}$. Recall this empirical kernel matrix is defined on  the target data $\boldsymbol{X}^{\text{tgt}}$ and is the limit when 
the width and embedding dimension $n$ goes to $\infty$:
\begin{equation}
  \lim_{n \rightarrow \infty} \frac{1}{n} \left\langle
     \nabla_{\theta} f\left(\theta(t), x\right),
     \nabla_{\theta} f\left(\theta(t), x^{\prime}\right)\right\rangle.
\end{equation}
We use standard parameterization in our proof, but note that the proof is almost the same for NTK parameterization with minor changes.

\subsection{Extension to Transformer Decoder Based Network}

\subsubsection{Proof Overview}

As discussed in the main paper, we consider two fine-tuning settings:
\begin{enumerate}
    \item $f_{\text{FT}}$ and $f_{\text{PT}}$ share the same structure
    \item An additional linear head can be appended to hidden layers of  $f_{\text{PT}}$ 
\end{enumerate}

Consider setting 1 first, 
to extend results in \cite{sunarxiv} for fully-connected networks, we need to prove two main conditions: 
\begin{enumerate}[label=\roman*)]
\item The Jacobian is locally Lipschitz under random Gaussian initialization  
     \item The  NTK at  initialization, i.e., pretrained initialization in our setup, 
     converges in probability to some analytic kernel $\Theta_0$ as $n$ goes to $\infty$.
\end{enumerate}

Suppose we have already proven  conditions i). 
Then we use the same proof strategy as in \cite{sunarxiv,sunaistats} to derive the final  mean prediction error convergence bound.  The major steps of the proof strategy in  \cite{sunarxiv,sunaistats} is  as follows:
\begin{enumerate}
    \item With condition i) 
    and the fact that the NTK of the network we consider converges in probability to some static kernel 
         at random initialization  as $n$ goes to $\infty$ \citep{yang2020tensor2}\footnote{\cite{yang2020tensor2}  states a general result for any network structure with random initialization.}, we can follow the same strategy of Theorem
G.1 in \cite{linearnn} to get the following results under random initialization: 
\begin{theorem}[NTK convergence at random initialization]
    Under assumptions 4.1-4.7 in the main text, for $\gamma>0$ and   $\eta_0<\eta_{\text {critical }}$, there exist $R_0>0, M\in \mathbb{N}$ and $L>1$, such that for every $m \geq M$, the following holds with probability at least $\left(1-\gamma\right)$ starting from $f_{\text{PT}}$ or $f^k_{\text{PT}}$ with an additional linear head added
    when applying gradient descent with learning rate $\epsilon=\frac{\eta_0}{m}$ considering standard parameterization, 
\begin{equation}
    \left\{\begin{array}{l}
\left\|g\left(\theta_\tau\right)\right\|_2 \leq\left(1-\frac{\eta_0 N_1 \lambda_{\min }}{3}\right)^\tau R_1 \\
\sum_{j=1}^\tau\left\|\theta_j-\theta_{j-1}\right\|_2 \leq \frac{\eta_0 \tilde{K} R_1}{\sqrt{n}} \sum_{j=1}^\tau\left(1-\frac{\eta_0 N_2\lambda_{\min }}{3}\right)^{j-1} \leq \frac{3 \tilde{K} R_0}{N_1\lambda_{\min }} n^{-\frac{1}{2}}
\end{array}\right.
\label{pretainntk}
\end{equation}
and
\begin{equation}
    \sup_{\tau}\left\|K_{\text{random\_init}}-K_{\tau}\right\|_F \leq \frac{6 L^3 R_1}{{N_1}^2 \lambda_{\min }} n^{-\frac{1}{2}}
    \label{ptntk}
\end{equation}
where $K_0$ is the empirical kernel matrix induced by $f_\text{random\_init}$ and $K_\tau$ is the one induced by $f_{\tau}$ using pre-trained  data $\boldsymbol{X}^{\text{src}}$. 
\end{theorem}
The second inequality in equation (\ref{pretainntk}) shows that the parameter of $f_{\text{PT}}$ lies within $O(n^{-1/2})$ of some random initialization as illustrated in Figure \ref{proofsketchpic}.  
Note that the main difference between the above theorem and Theorem G.1 in \cite{linearnn} is that our result applies to network architectures containing transformer blocks, whereas Theorem G.1 in \cite{linearnn} is restricted to fully connected networks. 
 

\item  
We use the fact that the parameters of $f_{\text{PT}}$ are within $O(n^{-1/2})$
of some random initialization, which is derived in previous step , along with the local Lipschitz property of the Jacobian under random normal initialization, to show that the Jacobian is also locally Lipschitz under pre-trained initialization. This reasoning is illustrated in Figure~\ref{proofsketchpic} and is essentially the same as the proof of Lemma~8 in \cite{sunarxiv}.

\item  Because  $K_{\text{random\_init}}$  converges in probability to some static kernel \citep{yang2020tensor2} and equation (\ref{ptntk}),  condition ii) is satisfied. The key observation here is that for fine-tuning task we replace the data, but this is fine, as we assume both source dataset and target dataset are bounded, so even though the empirical NTK is evaluated on $\boldsymbol{X}^{\text{src}}$, 
the NTK convergence bound still holds for target dataset with some uniform constants changed. 
The detailed discussion regarding this can found in the proof of Theorem 3 in \cite{sunarxiv}. 
Then we can use the same proof technique in the proof of  Theorem 3 and 4 in \cite{sunarxiv} to derive the NTK convergence under pretrained 
initialization (Theorem \ref{suppntkconver}) and the network linearization (Lemma \ref{supplinlemma}) in our setup.

\item The final step is to derive a bound on the mean prediction error under the optimal stopping time $\widehat{T}_{\text{op}}$ by following the same line of reasoning as in Theorem~2 and Corollary~2 of \cite{sunarxiv}. The proof techniques developed in \cite{sunarxiv} extend the methods of \cite{esnonpara} to enable a precise analysis of early stopping behavior in neural networks.


\end{enumerate}



For setting 2, we need to  we need to do a bit more work to handle the mix of random and non-random  initialization, which is slightly different from the proof for setting 1. 
Additional proof details will be  presented in later sections.


\begin{remark}
\label{giaremark}
To apply the NTK results from \citep{yang2020tensor2}, we must first verify the following key assumption, which forms the theoretical foundation of those results:
    \begin{assumption}[Simple GIA\footnote{Gradient independence assumption: for any matrix $W$, we assume
$W^{\top}$
used in backprop is independent from $W$ used in the forward pass.} Check]
\label{GIAass}
    At initialization, the output layer of $f_{\text{PT}}$ is
sampled independently and with zero mean from all other parameters and is not used anywhere else
in the interior of the network.
\end{assumption}
Because our setting excludes architectures with shared weights (e.g., RNNs), the assumption holds automatically, allowing us to safely apply the results from 
\cite{yang2020tensor2}.

\end{remark}


\subsubsection{Local Lipschitzness for setting 1}
We use $\theta$ to denote the collections of all parameters for the network under consideration.
And we use $\mathcal{X}$ to denote the training dataset, which may be a general dataset and is not restricted to the pre-trained source dataset. 
And for $l \geq 1$, let
\begin{equation}
    \begin{aligned}
& \delta^l(\theta, x):=\nabla_{f^l(\theta, x)} f^{L+1}(\theta, x) \in \mathbb{R}^{n} \\
& \delta^l(\theta, \mathcal{X}):=\nabla_{f^l(\theta, x)} f^{L+1}(\theta, \mathcal{X}) \in \mathbb{R}^{|\mathcal{X}| \times (n  \times    | \mathcal{X}| )}
\end{aligned}
\end{equation}
  Recall that $f^l(\theta, x)$ is the pre-activation for layer $l$.



The statement of the Lemma is as follows:
\begin{lemma}[Standard parameterization]
There is a $\tilde{K} >0$ such that for every $C>0$, with high probability over random initialization the following holds:
\begin{equation}
   \left\{\begin{array}{ll}  \frac{1}{\sqrt{n}}
\|J(\theta)-J(\tilde{\theta})\|_F & \leq \tilde{H} \|\theta-\tilde{\theta}\|_2 \\

 \frac{1}{\sqrt{n}}
\|J(\theta)\|_F & \leq H
\end{array}, \quad \forall \theta, \tilde{\theta} \in B\left(\theta_0, C n^{-1/2}\right)\right.
\end{equation}
where $\tilde{H}$ means a quantity that is $O\left(\log^{\alpha}(n)\right)$.
\label{lipmain}
\end{lemma}

\begin{proof}
The core idea of proving  this local lipschtiz is computing the gradient for each weight matrix and bound them directly. We adopt similar proof strategy as the proof of Lemma 1 in \cite{linearnn}.

We first consider adding only one attention layer without the residual connection (the ResNet structure). There are two main steps, the first is to bound the Jacobian of fully connected weights after introducing attention layer. The second is to bound to the Jacobian of attention weights 
themselves. 

The proof of this Lemma relies on the below theorem
to bound the operator norm of weight matrices:
\begin{theorem}[Corollary 5.35 \citep{vershynin2010} ] Let $A=A_{N, n}$ be an $N \times n$ random matrix whose entries are independent standard normal random variables. Then for every $t \geq 0$, with probability at least $1-2 \exp \left(-t^2 / 2\right)$ one has
\begin{equation}
    \sqrt{N}-\sqrt{n}-t \leq \lambda_{\min }(A) \leq \lambda_{\max }(A) \leq \sqrt{N}+\sqrt{n}+t.
\end{equation}
\end{theorem}
Let $\theta = \{  W^l, b^l, 
W^{l,O}, W^{l,Q}, W^{l,K}, W^{l,V}
\}$ and $\tilde{\theta} = \{  
\tilde{W}_l, \tilde{b}^l, 
\tilde{W}^{l,O}, \tilde{W}^{l,Q}, \tilde{W}^{l,K}, 
\tilde{W}^{l,V}
\}$ be any two points in
$B\left(\theta_0, C\right)$. 
By the above theorem  and the
triangle inequality, we have w.h.p. over random initialization for the first layer weight matrices the operator norm $\| \cdot \|_{\mathrm{op}}$ is bounded by  $3 \sigma_w \frac{\sqrt{n}}{\sqrt{n_0}}$, 
for $2 \leq l \leq L+1$,  $\| \cdot \|_{\mathrm{op}}$ is bounded by  $3 \sigma_w$. 

\paragraph{Fully connected weights.} To enable the same induction step for the fully connected layer weights, we first need to ensure that w.h.p. the following holds for 
the attention layer:
\begin{equation}
    n^{-\frac{1}{2}}\left\|f^l(\theta, \mathcal{X})\right\|_2, \quad\left\|\delta^l(\theta, \mathcal{X})\right\|_2  \leq K_1
\end{equation}
\begin{equation}
n^{-\frac{1}{2}}\left\|f^l(\theta, \mathcal{X})-f^l(\tilde{\theta}, \mathcal{X})\right\|_2 \leq K_1\|\tilde{\theta}-\theta\|_2 
\label{lpatt}
\end{equation}
\begin{equation}
    \left\|\delta^l(\theta, \mathcal{X})-\delta^l(\tilde{\theta}, \mathcal{X})\right\|_2  \leq \tilde{K}_2\|\tilde{\theta}-\theta\|_2
    \label{deltaatt}
\end{equation}
where $\theta$ and $\tilde{\theta}$ is any two points in
$B\left(\theta_0, C\right)$, $K_1$ is some universal constants depend on $|\mathcal{X}|$ and $L$, $\tilde{K}_2$ is some universal constant of order $O(\sqrt{\log(n}))$.  Note that in the following proofs, any constants related to $|\mathcal{X}|$ and $L$  are treated as $\mathcal{O}(1)$, we are only interested in and should be careful about the constants that grow w.r.t. $n$.

We start by bounding $\left\|f^l(\theta, x)\right\|_2$. 
Recall for each head, the output is calculated  as 
\begin{equation}
    f^{l h}(x)=\zeta\left(\frac{1}{d^{l, G}} g^{l-1}(x) W^{l h, Q}\left(g^{l-1}(x) W^{l h, K}\right)^{\top}\right) g^{l-1}(x) W^{l h, V}
\end{equation}
We have known that 
\begin{itemize}
    \item the softmax matrix is of finite dimension and each entry is bounded
    \item by the induction step $\left\|g^{l-1}(x)\right\|_2\leq n^{1/2} K_1$\footnote{Assume all preceding layers are fully connected layer and use the same induction step as the (S85) in \cite{linearnn}.}
    \item $\| W^{l h, V}  \|_{\mathrm{op}}\leq 3\sigma_w$  w.h.p.
\end{itemize}
then it is easy to see we can bound $\left\|f^l(\theta, x)\right\|_2$ as desired.

Next we bound $\delta^l(\theta, x)$. By back-prop recursion we know that 
\begin{equation}
    \delta^l(\theta, x)=\left(\delta^{l+1}(\theta, x) {W^{l+1}}^{\top}\right) \odot \phi^{\prime}\left(f^l(\theta, x)\right) 
\end{equation}
again, use the induction step results in (S85) in \cite{linearnn}, it is not hard to bound $\delta^l(\theta, x)$. 

We then proceed to show (\ref{lpatt}). By the definition of multi-head attention, we have 
\begin{equation}
    \begin{aligned}
        \|f^l(\theta, \mathcal{X})-f^l(\tilde{\theta}, \mathcal{X})\|_2 &= \|    \sum_h f^{lh}((\theta, x)W^{lh,O} - 
         \sum_h f^{lh}(\tilde{\theta}, \mathcal{X})\tilde{W}^{lh,O} 
        \|_2 \\
        &\leq \sum_h \|    
        f^{lh}((\theta, \mathcal{X})W^{lh,O} - f^{lh}(\tilde{\theta}, \mathcal{X})\tilde{W}^{lh,O} 
        \|_2
    \end{aligned}
\end{equation}
Since we consider the number of head to be finite, it suffices to prove the lipschtzness for one single head.  After using the add and subtract tricks
and all the proved bounds, the major thing left to prove is the lipschtzness of the pre-softmax quantity of the attention layer. 
Denote $A = \frac{1}{d^{l,G}} QK^T$, where
$Q = g^{l-1}(x)W^{lh,Q}, K = g^{l-1}(x)W^{lh,K}$. Let $\Delta Q := g^{l-1}W^{lh,Q} - \tilde{g}^{l-1}\tilde{W}^{lh,Q}$, we have
\begin{equation}
\begin{aligned}
        \| \Delta Q \|_2 &\leq \|g^{l-1}\|_2 \| W^{lh,Q} - \tilde{W}^{lh,Q} \|_F + \|W^{lh,Q}\|_{op} 
    \| g^{l-1} - \tilde{g}^{l-1}  \|_2  \\
    &\leq  O(n^{1/2}) \|\theta - \tilde{\theta} \|_2
\end{aligned}
\end{equation}
where we use the induction result on 
$\| g^{l-1} - \tilde{g}^{l-1}  \|$, which is simply the (S85) in \cite{linearnn}. Similarly, for $\Delta K$, we have the same bound. Combine, we have
\begin{equation}
    \begin{aligned}
        \| \Delta A \|_2 &\leq 
        \frac{1}{d^{l,G}}  \| \Delta Q \|_2 \cdot O(n^{1/2} )\cdot \|W^{lh,K}\|_{op} 
        + \frac{1}{d^{l,G}}  \| \Delta K \|_2 \cdot O(n^{1/2} )\cdot \|W^{lh,Q}\|_{op}
         \\
    &\leq  \frac{O(n)}{d^{l,G}} \|\theta - \tilde{\theta} \|_2 \\ & = O(1) \|\theta - \tilde{\theta} \|_2
\end{aligned}
\end{equation}
where we use the fact that $ \|g^{l-1}\|_2 \|$ is $O(n^{1/2})$
 and 
$d^{l,G} = n$. 

Finally, we prove (\ref{deltaatt}). The proof is the similar to the proof of (S86) in \cite{linearnn}. The only difference is that during the induction step and while bounding this term, we will encounter a term as below
\begin{equation}
    \delta^{l+1}(\tilde{\theta}) {W^{l+1}}^{\top}  \odot \left(  \phi^{\prime}\left(f^l(\theta, \mathcal{X}) \right)   -
    \phi^{\prime}\left(f^l(\tilde{\theta}, \mathcal{X}) \right
    ) \right) \quad   (\textcolor{red}{A})
\end{equation}
We bound it as follows 
\begin{equation}
    \begin{aligned}
   \| (\textcolor{red}{A}) \|_2 &\leq   
    \| \delta^{l+1}(\tilde{\theta})
     {W^{l+1}}^{\top}  \|_{\infty}  \| \phi^{\prime}\left(f^l(\theta, \mathcal{X}) \right)   -
    \phi^{\prime}\left(f^l(\tilde{\theta}, \mathcal{X}) \right
    ) \|_2
    \\  
    &\leq O\left(
    \|\delta^{l+1}\|_2
    \frac{\sqrt{\log(n)}}{n^{1/2}}\right)  O(n^{1/2}) \|\theta - \tilde{\theta}\|_2 \\
    &= \tilde{O}(1) \|\theta - \tilde{\theta}\|_2
    \end{aligned}
\end{equation}
where we use induction results regarding $\delta^l$ and $f^l$ and assumption about the activation function. 
For $\| \delta^{l+1}
     {W^{l+1}}^{\top}  \|_{\infty}$, observe that condition on $\delta^{l+1}$, each entry is a Gaussian random variable $\mathcal{N}(0, \frac{\| \delta^{l+1}\|_2^2}{n})$, so we can  
bound it using the tail probability of Gaussian  variable to get the $O\left(
\| \delta^{l+1}\|_2
\frac{\sqrt{\log(n)}}{n^{1/2}}\right)$ bound. Note that here we need to make use of the GIA assumption mentioned in Remark \ref{giaremark}. 

Besides these, we need to bound $\frac{\partial f^l}{\partial g^{l-1}}$ and prove its lipstchzness to enable the same induction proof on fully connected layer weights. This is crucial because for the layer before the attention layer
\begin{equation}
    \delta^{l-1}(\theta)  = \delta^l(\theta) \frac{\partial f^l}{\partial g^{l-1}}.
\end{equation}
We define the intermediate quantities to facilitate the 
calculation 
of this gradient:
\begin{equation}
    \begin{aligned}
& Q_h=g^{l-1} W^{lh,Q}, \quad K_h= g^{l-1}W^{lh,K}, \quad V=X W^{lh,V}, \\
& S_h=\frac{1}{d^{l,G}} Q_h K_h^{\top}, \quad P_h=\operatorname{softmax}_{\mathrm{rows}}(S_h), \quad A_h=P_h V_h,  \\
& \Lambda_h=\operatorname{block\_ diag}\left(\Lambda_{h,0}, \ldots, \Lambda_{h,T-1}\right) \in \mathbb{R}^{T^2 \times T^2}, \quad \Lambda_{h,i}=\operatorname{diag}\left(P_{h,i:}\right)-P_{h,i:} P_{h,i:}^{\top}.
\end{aligned}
\label{nta}
\end{equation}
Then we have for each head, we denote $\frac{\partial f^{lh}}{\partial g^{l-1}} \in \mathbb{R}^{\left(T n\right) \times\left(T n\right)}$ as $J^h$
\begin{equation}
J^h
=W_h^{V \top} \otimes P_h+\frac{1}{d^{l,G}}\left(I_{n} \otimes V_h^{\top}\right) \Lambda_h\left[\left(K_h \otimes I_T\right)\left(W^{lh, Q T} \otimes I_T\right)+\left(I_T \otimes Q_h\right)\left(W^{lh, K \top} \otimes I_T\right)\right] 
.
\end{equation}
Then the final gradient is 
\begin{equation}
    \frac{\partial f^l}{\partial g^{l-1}} = \left(W^{l,O \top} \otimes I_T\right)\left(\bigoplus_{h=1}^H J^h\right)
    \in \mathbb{R}^{\left(T n\right) \times\left(T n\right)}
\end{equation}
Use same argument as above, it is not hard to show that 
\begin{equation}
    \begin{aligned}
    \left\| \frac{\partial f^l}{\partial g^{l-1}} \right\|_F \leq O(1), 
  \left\| 
  \frac{\partial f^l}{\partial g^{l-1}} (\theta,x)  -  \frac{\partial f^l}{\partial g^{l-1}} (\tilde{\theta},x)
  \right\|_F   \leq O(1) \| \theta - \tilde{\theta}\|_2.
    \end{aligned}
\end{equation}

Note that we use pooling operation after the attention layer, since we assume they are simple deterministic reduction operation, and we only operate on finite dimension, this would not affect the above proved results.

\paragraph{Attention weights.} 
For each attention weights, the goal is to bound the Jacobian of itself and also prove the lipstchzness.

We start with $W^{l,O}$. Note that
\begin{equation}
    J(W^{l,O}) = \left[f^{l 1}(x), \ldots, f^{l d^{l, H}}(x)\right]^T \delta^l(\theta,x) 
\end{equation}
Then it is not hard to prove that
\begin{equation}
\begin{aligned}
 &\left\| \left[f^{l 1}(\mathcal{X}), \ldots, f^{l d^{l, H}}(\mathcal{X})\right]^T \delta^l(\theta,\mathcal{X})   \right\|_F \leq O(n^{1/2}) \\
& \left\| J(W^{l,O})(\theta,\mathcal{X}) - J(W^{l,O})(\tilde{\theta},\mathcal{X}) \right\|_F \leq O(n^{1/2} \sqrt{\log (n)}) \| \theta - \tilde{\theta}\|_2.
 \end{aligned}
\end{equation}

Next for $W^{lh,V}$, we have
\begin{equation}
    J(W^{lh,V}) = {g^{l-1}}^{\top} P_h^T \frac{\partial f^{L+1}}{\partial f^{lh}} W^{lh,O T}
\end{equation}
where $\frac{\partial f^{L+1}}{\partial f^{lh}} = \delta^l(\theta, x) {W^{lh,O}}^{\top}$.

Similarly, we can show that
\begin{equation}
\begin{aligned}
 &\left\| J(W^{lh,V})   \right\|_F \leq O(n^{1/2}) \\
& \left\| J(W^{lh,V})(\theta,x) - J(W^{lh,V})(\tilde{\theta},x) \right\|_F \leq O(n^{1/2} \sqrt{\log (n)}) \| \theta - \tilde{\theta}\|_2.
 \end{aligned}
\end{equation}

Finally, for query and key matrices, we have
\begin{equation}
\begin{aligned}
& J(W^{lh,Q}) = {g^{l-1}}^{\top} \cdot\left(\frac{1}{ d^{l,G}} \cdot
{\frac{\partial P_h}{\partial S_h}}^{\top} \cdot \frac{\partial f^{L+1}}{\partial P_h}
\cdot K^h\right) \\
& J(W^{lh,K}) ={g^{l-1}}^{\top} \cdot\left(\frac{1}{d^{l,G}} \cdot {\frac{\partial f^{L+1}}{\partial P_h}}^{\top} \cdot {\frac{\partial P_h}{\partial S_h}}^{\top}
\cdot Q^h\right)
\end{aligned}
\end{equation}
where $\frac{\partial f^{L+1}}{\partial P_h} = \frac{\partial f^{L+1}}{\partial f^{lh}} {V^h}^{\top}$
and ${\frac{\partial P_h}{\partial S_h}}^{\top}$ is simply $\Lambda_h$ defined in (\ref{nta}).

We can show that the following holds 
for these two types of weights 
by similar reasoning as above:
\begin{equation}
\begin{aligned}
 &\left\| J(W^{lh,Q})   \right\|_F \leq O(n^{1/2}) \\
& \left\| J(W^{lh,Q})(\theta,\mathcal{X}) - J(W^{lh,Q})(\tilde{\theta},\mathcal{X}) \right\|_F \leq O(n^{1/2} \sqrt{\log (n)}) \| \theta - \tilde{\theta}\|_2 \\
 &\left\| J(W^{lh,K})   \right\|_F \leq O(n^{1/2}) \\
& \left\| J(W^{lh,K})(\theta,\mathcal{X}) - J(W^{lh,K})(\tilde{\theta},\mathcal{X}) \right\|_F \leq O(n^{1/2} \sqrt{\log (n)}) \| \theta - \tilde{\theta}\|_2.
 \end{aligned}
\end{equation}


The results of Lemma \ref{lipmain} follows from the above estimates, we just sum all together to get the desired bound
\begin{equation}
\begin{aligned}
\frac{1}{\sqrt{n}}
\|J(\theta)\|_F^2&=\underbrace{\sum_l\left\|\frac{1}{\sqrt{n}}J\left(W^l\right)\right\|_F^2+\left\|J\left(b^l\right)\right\|_F^2}_{\text{fully connected layers}} + \underbrace{
\sum_l \left\|\frac{1}{\sqrt{n}}J\left(W^{l,O}\right)\right\|_F^2 + \sum_{\substack{l, h,  \bullet = \{K, Q, V\}}}
 \left\|\frac{1}{\sqrt{n}}J\left(W^{lh, \bullet}\right)\right\|_F^2}_{\text{attention layer weights}}
\\
& = \frac{1}{n}
\sum_l \sum_{x \in \mathcal{X}}\left\|g^{l-1}(\theta, x) \delta^l(\theta, x)^T\right\|_F^2+\left\|\delta^l(\theta, x)^T\right\|_F^2  + O(1)  \\
& \leq\frac{1}{n}  \sum_l \sum_{x \in \mathcal{X}}\left(1+\left\|g^{l-1}(\theta, x)\right\|_F^2\right)\left\|\delta^l(\theta, x)^T\right\|_F^2  + O(1)  \\
&  \leq \sum_l \frac{1}{n}\left(1+K_1^2 n\right) \sum_x\left\|\delta^l(\theta, x)^T\right\|_F^2 + O(1) \\
& \leq \sum_l \frac{1}{n} K_1^2\left(1+K_1^2 n\right) + O(1)  \\
 &\leq  2(L+1) K_1^4   + O(1)  \\
 & = O(1)
\end{aligned}
\end{equation}
Similarly, for the lipschitzness
\begin{equation}
    \begin{aligned}
 \|J(\theta)-J(\tilde{\theta})\|_F^2 
&=  \underbrace{\sum_l \sum_{x \in \mathcal{X}}\left\|\frac{1}{\sqrt{n}}g^{l-1}(\theta, x) \delta^l(\theta, x)^T- \frac{1}{\sqrt{n}} g^{l-1}(\tilde{\theta}, x) \delta^l(\tilde{\theta}, x)^T\right\|_F^2+\left\|\delta^l(\theta, x)^T-\delta^l(\tilde{\theta}, x)^T\right\|_F^2}_{\text{fully connected layers}} \\
& +  \underbrace{
\sum_l \frac{1}{n} \left\|J\left(W^{l,O},\theta\right)
-J\left(W^{l,O},\tilde{\theta}\right)
\right\|_F^2 + \sum_{\substack{l, h,  \bullet = \{K, Q, V\}}} \frac{1}{n}
 \left\|J\left(W^{lh, \bullet},\theta\right)
 - J\left(W^{lh, \bullet},\tilde{\theta}\right)
 \right\|_F^2
}_{\text{attention layer weights}}
\\
&\leq  \left(\sum_l 4K_1^2
\tilde{K}_1^2 +\tilde{K}_1^2\right)\|\theta-\tilde{\theta}\|_2 
+ O(\log(n)) \|\theta-\tilde{\theta}\|_2 
\\
&\leq  3(L+1) K_1^2 \tilde{K}_1^2 \|\theta-\tilde{\theta}\|_2  + O(\log(n)) \|\theta-\tilde{\theta}\|_2  \\
&= O(\log(n)) \|\theta-\tilde{\theta}\|_2.
\end{aligned}
\end{equation}

We then extend the results to transformer block. By definition, the transformer output is
\begin{equation}
    \begin{aligned}
        z^{l} &=  g^{l-1} + \text{multi-head attention layer}(g^{l-1}) \\
        f^l(x) &=  z^{l}  + \text{fully connected layer}(z^{l}).
    \end{aligned}
\end{equation}
Since we already prove the results for multi-head attention layer, by the above formula, it is trivial to extend all these results to transformer block using the same proof strategy.

Also, we can easily extend the above results to network with multiple transformer blocks by induction. The only left component in the transformer decoder is the LayerNorm. We consider the specific steps to deal with is as below.

\paragraph{LayerNorm.}  For theoretical convenience, we give the proof for the case where the sequence length is 1, since $T$  is finite, the result can be easily extended to standard situation. Given an LayerNorm input vector $x \in \mathbb{R}^n$ (for one token), LayerNorm first computes mean and variance as follows:
\begin{equation}
     \mu = \frac{1}{n} \sum_{i=1}^n x_i, \  \  \sigma^2 = \frac{1}{n } \sum_{i=1}^n (x_i  - \mu)^2
\end{equation}
and then normalize and add learnable parameters:
\begin{equation}
     \hat{x}_i=\frac{x_i-\mu}{\sqrt{\sigma^2+\epsilon}}, \  \ y_i=\gamma_i \hat{x}_i+\beta_i.
\end{equation}
The final output is given by 
\begin{equation}
    y=\gamma \odot \frac{x-\mu}{\sqrt{\sigma^2+\epsilon}}+\beta
\end{equation}
where $\gamma, \beta \in \mathbb{R}^n$ are learnable parameters, and 
$\epsilon$
is a small constant for numerical stability. Note $x$ is the output of some preceding layers. 
After introducing LayerNorm to the network, the Jacobian of some of the weights will be affected, so
in order to maintain
 the same Lipschitz norm, we need to prove the following two main conditions hold w.h.p.:
\begin{enumerate}
    \item 
    \begin{equation}
         \left\| 
    \frac{\partial f^{L+1}}{ \partial x} 
    \right\|_2\leq \mathcal{O}(\sqrt{\log n})
    \label{lncon1}
    \end{equation}

    \item  
    \begin{equation}
        \left\|  \frac{\partial f^{L+1}}{ \partial x}(
       \theta 
        )  - \frac{\partial f^{L+1}}{ \partial x}(
       \tilde{\theta }
        ) 
        \right\|_2   \leq \mathcal{O}(\sqrt{\log n}) \| \theta - \tilde{\theta }   \|_2
        \label{lncon2}
    \end{equation}
\end{enumerate}
 We start by proving equation (\ref{lncon1}). Let $\delta_i$  denote upstream gradient $\frac{\partial f^{L+1}}{\partial y_i}$, and $\tilde{u}_i$ denote $u_i \gamma_i$, then we can write:
 \begin{equation}
     \frac{\partial f^{L+1}}{\partial x} = \frac{1}{\sqrt{\sigma^2 + \epsilon}} 
\left[  \tilde{\delta} - \text{mean}(\tilde{\delta}) \cdot \mathbf{1}
 -  \text{mean}( \tilde{\delta} \odot \tilde{x} )\cdot \tilde{x}
\right]
 \end{equation}
where $\tilde{x}, \delta, \tilde{\delta} \in \mathbb{R}^n $
,$\mathbf{1}$ is an all-one vector of length $n$, and 
$\text{mean}(v) = \frac{1}{n} \sum_j v_j$.
Observe that 
\begin{equation}
    \frac{1}{\sqrt{\sigma^2 + \epsilon}} \leq \frac{1}{\sqrt{\epsilon}} 
    \label{first_69}
\end{equation}
and 
\begin{equation}
    \| \tilde{\delta}\|_2 =  \| \gamma \|_{\infty} \| \delta  \|_2    \leq \mathcal{O}(\sqrt{\log n}) \mathcal{O}(1)
    \label{sec_69}
\end{equation}
where we use the fact that $\gamma$ are initialized with random Gaussian and the induction result from upstream gradient. 

Next for $\text{mean}(\tilde{\delta}) \cdot \mathbf{1}$, we apply Cachy-Schwarz inequality:
\begin{equation}
    | \text{mean}(\tilde{\delta}) =
| \frac{1}{n} \mathbf{1}^{\top} \tilde{\delta} | \leq 
\frac{1}{n} \|\mathbf{1} \|_2 \|\tilde{\delta} \|_2  = \frac{1}{\sqrt{n}} \|\tilde{\delta} \|_2
\end{equation}
This leads to 
\begin{equation}
 \| \text{mean}(\tilde{\delta}) \cdot \mathbf{1} \|_2 \leq \| \tilde{\delta} \|_2
 \leq  \| \gamma \|_{\infty} \| \delta  \|_2    \leq \mathcal{O}(\sqrt{\log n}) \mathcal{O}(1) .
 \label{third_69}
\end{equation}

For the last term, let $T:= \text{mean}( \tilde{\delta} \odot x )\cdot \tilde{x}$, then we have
\begin{equation}
    \| T \|_2 = \| \tilde{x} \|_2 |          
    \text{mean}( \tilde{\delta} \odot \tilde{x} )
    |
\end{equation}
and 
\begin{equation}
    |   \text{mean}( \tilde{\delta} \odot \tilde{x} ) | =   | \frac{1}{d}     \langle \tilde{\delta} ,  \tilde{x} \rangle|       \leq \frac{1}{d}  \|  \tilde{\delta} \|_2 \|\tilde{x} \|_2.
\end{equation}
Combine, we have
\begin{equation}
\begin{aligned}
     \|T\|_2 &\leq \|\tilde{x}  \|_2 \cdot \frac{1}{n} \|   \tilde{\delta}  \|_2 \|\tilde{x}  \|_2 \\
     &= \|   \tilde{\delta}  \|_2
     \cdot 
     \frac{1}{n} \|\tilde{x}  \|_2^2  \\
     &= \|   \tilde{\delta}  \|_2 \cdot 
     \frac{1}{n} \sum_j \tilde{x}_j^2  \\
     &=\|   \tilde{\delta}  \|_2 \cdot  \frac{1}{n} \sum_j  \frac{
     (x_j-\mu)^2
     }{\sigma^2+\epsilon}\\
     &= \|   \tilde{\delta}  \|_2 \cdot 
    \frac{1}{\sigma^2+\epsilon} \cdot \frac{1}{d} \sum_j  (x_j - \mu)^2 
    \\
    &= \|   \tilde{\delta}  \|_2 \cdot \frac{1}{\sigma^2+\epsilon} \cdot \sigma^2 \\  &\leq 
    \|   \tilde{\delta}  \|_2 \\
    &  \leq  \| \gamma \|_{\infty} \| \delta  \|_2    \leq \mathcal{O}(\sqrt{\log n}) \mathcal{O}(1).
    \label{fourth_69}
\end{aligned}
\end{equation}
Combine results (\ref{first_69})
 (\ref{sec_69})
  (\ref{third_69})  (\ref{fourth_69}), we can get   (\ref{lncon1}).

Now we turn to prove (\ref{lncon2}). We use the same add and subtract trick which is commonly used in this proof. And there is a term related to $\sigma$, and we need to show that 
\begin{equation}
    \left| \frac{1}{\sqrt{\sigma^2 + \epsilon}} -    
    \frac{1}{\sqrt{\tilde{\sigma}^2 + \epsilon}} 
    \right| \leq O(1) \| \theta - \tilde{\theta} \|_2 .
\end{equation}
Note that there are some other terms we need to bound, but we basically use the same proof trick so we omit here. 
We first use mean value theorem to get
\begin{equation}
    \left| \frac{1}{\sqrt{\sigma^2 + \epsilon}} -    
    \frac{1}{\sqrt{\tilde{\sigma}^2 + \epsilon}} 
    \right| \leq \frac{1}{2\epsilon^{\frac{3}{2}}} |\sigma^2 - \tilde{\sigma}^2  |.
\end{equation}
Next we write $\sigma$ with projection matrix $P := \mathbf{1} - \frac{1}{n} \mathbf{1} \mathbf{1}^{\top}$, then
we have
\begin{equation}
    \sigma^2(x) = \frac{1}{n} \sum_{j} (x_j - \mu(x))^2 = \frac{1}{n} \|Px\|_2^2. 
\end{equation}
Then we have
\begin{equation}
    \begin{aligned}
  |\sigma^2 - \tilde{\sigma}^2  | &= \frac{1}{n} \left|  \|Px\|_2^2 - \|P \tilde{x} \|_2^2     \right|      \\
  &= \frac{1}{n} \left| 
  {(Px + P\tilde{x})}^{\top} (  Px - P\tilde{x}    )
  \right| \\
  & \leq  \frac{1}{n}\| P(x-\tilde{x})\|_2(  \|
  Px
  \|_2 + \| P\tilde{x}  \|_2 ) \\
  & \leq  \frac{1}{n} \| x -  \tilde{x}       \|_2
  (     \|
  Px
  \|_2 + \| P\tilde{x}  \|_2        )
    \end{aligned}
\end{equation}
where we use the fact that $P$ is orthogonal projection and $\|Pz\|_2 \leq \|z\|_2$.
Next since $\|Px\|_2 = \sqrt{n\sigma^2(x)}$, we can write
\begin{equation}
     |\sigma^2 - \tilde{\sigma}^2  | \leq \frac{1}{\sqrt{n}} 
\| x - \tilde{x}\|_2 \left(
\sqrt{\sigma^2(x)}
  +   \sqrt{\sigma^2(\tilde{x})}       \right).
\end{equation}
Use the fact that $\sigma^2(x) \leq \frac{1}{n} \|x\|_2^2$ and $\|x\|_2 \leq O(\sqrt{n})$, we have
\begin{equation}
    \sqrt{\sigma^2(x)}
  +   \sqrt{\sigma^2(\tilde{x})} \leq \frac{2O(\sqrt{n})}{\sqrt{n}} = O(1)
\end{equation}
finally, we get
\begin{equation}
     |\sigma^2 - \tilde{\sigma}^2  | \leq  \frac{O(\sqrt{n})}
     {\sqrt{n}} \| \theta - \tilde{\theta} \|_2 = O(1) \| \theta - \tilde{\theta} \|_2.
\end{equation}
where we use the induction result w.r.t. $x$.

The gradient w.r.t. $\gamma$ and $\beta$ is given by 
\begin{equation}
\begin{aligned}
    \frac{\partial f^{L+1}}{\partial \gamma} &= 
    \frac{\partial f^{L+1}}{ \partial y}    
\odot 
    \tilde{x}
    \\
    \frac{\partial f^{L+1}}{\partial \beta} &=
     \frac{\partial f^{L+1}}{ \partial y}
\end{aligned}
\end{equation}
we can prove the desired bound for these gradients and their Lipschitzness with similar reasoning steps as above.

Note that here we consider full batch gradient descent, and we do not write out the summation over training samples $\mathcal{X}$  
explicitly 
as this will only add a constant related to $|\mathcal{X}|$ in the final bound, wich we simply view them as $O(1)$. 
And this concludes our proof.




\end{proof}

The local Lipschitzness for the NTK parameterization is as follows: 
\begin{lemma}[NTK parameterization]
There is a $\tilde{L} >0$ such that for every $C>0$, with high probability over random initialization $\theta_0$ the following holds:
\begin{equation}
   \left\{\begin{array}{ll}  
\|J(\theta)-J(\tilde{\theta})\|_F & \leq \tilde{H} \|\theta-\tilde{\theta}\|_2 \\

\|J(\theta)\|_F & \leq H
\end{array}, \quad \forall \theta, \tilde{\theta} \in B\left(\theta_0, C\right)\right.
\end{equation}
\label{lipmain}
\end{lemma}

\subsection{Proof for Setting 2}
We formally state the fine-tuning setting 2 as follows:
take the $k$-th hidden layer
of the pretrained network $f_{\text{PT}}$ as backbone denoted by $f_{\text{PT}}^k$, 
we append a linear head after it. We denote the parameter of this linear head as $\{W, b\}$, with $W_{i j} \sim \mathcal{N}\left(0, \sigma^2 / n\right), \  b \sim \mathcal{N}\left(0, \sigma^2\right)$.

\begin{lemma}
    There is a $L >0$ such that for every $C>0$, with high probability the following holds over the special mixed  initialization described above
    :
\begin{equation}
   \left\{\begin{array}{ll}  \frac{1}{\sqrt{n}}
\|J(\theta)-J(\tilde{\theta})\|_F & \leq \tilde{L} \|\theta-\tilde{\theta}\|_2 \\

 \frac{1}{\sqrt{n}}
\|J(\theta)\|_F & \leq L
\end{array}, \quad \forall \theta, \tilde{\theta} \in B\left(\theta_0, C n^{-1/2}\right)\right.
\end{equation}
where $\theta_0$ denote the initial parameter values of linear head, which is sampled from random Gaussian, and the backbone, which inherit
 the pretrained parameters.
\end{lemma}
\begin{proof}
    The proof idea is similar to the proof of Lemma 8 in \cite{sunarxiv}. Denote the parameter of the initial backbone part  as 
    $\theta^0_{\text{BackBone}}$ and the initial 
    parameter of linear head as $\{W^{r.i.}, b^{r.i.}\}$. 
So $\theta_0$ is $\{\theta^0_{\text{BackBone}}, W^0, b^0\}$. 

Let $\theta^{r.i.}_{\text{BackBone}}$ be the corresponding random normal
initialization  part for the backbone network.
Based on the corollary
 of Lemma \ref{lipmain}, the distance between them can be bounded as 
 \begin{equation}
     \| \theta^{r.i.}_{\text{BackBone}} - \theta^0_{\text{BackBone}}\| \leq C_1 n^{-\frac{1}{2}}
 \end{equation}

 Then for any $\theta \in B\left(\theta_0, C n^{-1/2}\right)$, we have w.h.p.
 \begin{equation}
 \begin{aligned}
     \| W - W^{r.i.}\| + \|b - b^{r.i.}\| + \| \theta_{\text{BackBone}} - \theta^{r.i.}_{\text{BackBone}}  \| &\leq 
 \| W - W^{r.i.}\| + \|b - b^{r.i.}\| \\ & + 
 \| \theta_{\text{BackBone}} - \theta^{0}_{\text{BackBone}}  \| +  \| \theta^{r.i.}_{\text{BackBone}} - \theta^0_{\text{BackBone}}\| \\
 &\leq  (C_1 + C) n^{-1/2}
 \end{aligned}
 \end{equation}
 Same argument holds for $\tilde{\theta}$. Therefore we have shown that both $\tilde{\theta}, \theta \in B\left(\theta^{r.i.}_0, C n^{-1/2}\right)$, where $\theta^{r.i.}_0$ is the random initialization of this network we have considered in this paper.  So make use of 
Lemma  \ref{lipmain} again, we concludes our proof.

\end{proof}

\begin{lemma}
    For this particular setting described above, the NTK converges to some analytical kernel as $n\rightarrow \infty$. 
\end{lemma}
\begin{proof}
    Denote the backbone part as  $f_{\text{BackBone}}$, then the NTK w.r.t the linear head $W$ is 
    \begin{equation}
        n^{-\frac{1}{2}} f_{\text{BackBone}}  \cdot n^{-\frac{1}{2}}f_{\text{BackBone}}^{\top}
    \end{equation}
According to  estimates (S86) in \cite{linearnn} estimates, we have
\begin{equation}
   n^{-\frac{1}{2}}  f_{\text{BackBone}}  = 
   \frac{
   x^0_{r.i.}
    }{ n^{\frac{1}{2}}} + {\mathcal{O}}\left(  \frac{\sqrt{\log n }}{n^{\frac{1}{2}}}\right)
\end{equation}
where $x^0_{r.i.}$ denote the 
random initialized version of 
$f_{\text{BackBone}}$.

NTK w.r.t.  $f_{\text{BackBone}}$ is 
\begin{equation}
\frac{1}{n}
    W J_{\text{BackBone}}(\theta_{\text{PT}})    
    J^{\top}_{\text{BackBone}}(\theta_{\text{PT}})    
    W^\top
\end{equation}
For $J_{\text{BackBone}}(\theta_{\text{PT}})$, we have
\begin{equation}
    J_{\text{BackBone}}(\theta_0) + \Delta
\end{equation}
where $\theta_0$ is the random 
initialization when training $f_{\text{PT}}$, and 
$\Delta$ is the differece between these two, i.e., 
$
J_{\text{BackBone}}(\theta_{\text{PT}}) - J_{\text{BackBone}}(\theta_0)
$. This term can also be controlled
by ${\mathcal{O}}\left(  \frac{\sqrt{\log n }}{n^{\frac{1}{2}}}\right)$
using because of the local  lipschtiz result we proved above. Putting these all together, the NTK here becomes 
\begin{equation}
\label{erreaf}
    \frac{1}{n} W J_{\text{BackBone}}(\theta_0)J^{\top}_{\text{BackBone}}(\theta_{0})    
    W^\top + \frac{x^0_{r.i.} {x^0_{r.i.}}^{\top}}{n} 
    + e 
\end{equation}
where the norm of error term $e$ vanishes to 0 with rate
${\mathcal{O}}\left(  \frac{\sqrt{\log n }}{n^{\frac{1}{2}}}\right)$.  

Note that the part excluding $e$ in (\ref{erreaf}) recovers the NTK for a random initialized network, 
then according to \cite{yang2020tensor2}, which proves the NTK for any network architecture under random normal initialization. 
And this concludes our proof.

\end{proof}
\subsection{Main results}
Based on the  the Local 
    Lipschitzness property proved above, we can get the following NTK convergence results under the two types of fine-tuning settings.

\begin{theorem}[NTK convergence]
\label{suppntkconver}
    Under assumptions 4.1-4.7 in the main text, for $\gamma>0$ and   $\eta_0<\eta_{\text {critical }}$, there exist $R_0>0, M\in \mathbb{N}$ and $L>1$, such that for every $m \geq M$, the following holds with probability at least $\left(1-\gamma\right)$ starting from $f_{\text{PT}}$ or $f^k_{\text{PT}}$ with an additional linear head added
    when applying gradient descent with learning rate $\epsilon=\frac{\eta_0}{m}$ considering standard parameterization, 
\begin{equation}
    \left\{\begin{array}{l}
\left\|g\left(\theta_\tau\right)\right\|_2 \leq\left(1-\frac{\eta_0 N_2 \lambda_{\min }}{3}\right)^\tau R_0 \\
\sum_{j=1}^\tau\left\|\theta_j-\theta_{j-1}\right\|_2 \leq \frac{\eta_0 \tilde{K} R_0}{\sqrt{n}} \sum_{j=1}^\tau\left(1-\frac{\eta_0 N_2\lambda_{\min }}{3}\right)^{j-1} \leq \frac{3 \tilde{K} R_0}{N_2\lambda_{\min }} n^{-\frac{1}{2}}
\end{array}\right.
\end{equation}
and
\begin{equation}
    \sup_{\tau}\left\|K_0-K_{\tau}\right\|_F \leq \frac{6 L^3 R_0}{{N_2}^2 \lambda_{\min }} n^{-\frac{1}{2}}
\end{equation}
where $K_0$ is the empirical kernel matrix induced by $f_{\text{PT}}$ and $K_\tau$ is the one induced by $f_{\tau}$ using target data $\boldsymbol{X}^{\text{tgt}}$. 
\end{theorem}
\begin{proof}
    See proof of Theorem 3 in \cite{sunarxiv}. 
    \end{proof}

\begin{lemma}[Lemma 4.1 in the main text]
\label{supplinlemma}
Under assumptions 4.1 - 4.7 in the main text,  we are able to linearize the fine-tuned network $f_{\text{PT}}$ around the pretrained model $f_{\text{PT}}$, i.e., we have
\begin{equation}
\begin{aligned}
      f_{\text{FT}} &=   f_{\text{PT}} + \Delta \theta \nabla f_{\text{PT}}(\theta) + e \\
      & = f_{\text{PT}}  + \sum_i  \zeta_i \mathbb{K} (\cdot,\mathbf{X}_i ) + e
\end{aligned}
\end{equation}
   where  the $L_2$-norm of the error term $e$ is bounded by
   $\tilde{\mathcal{O}}\left( \frac{1}{{n^l}^{1/2}}\right)$.
\end{lemma}
\begin{proof}
     This can be derived  based on the Local 
    Lipschitzness property obtained 
    for the particular network considered in this paper. See proof of Theorem 4 and Lemma 2 in \cite{sunarxiv}. 
\end{proof}

\begin{theorem}[Theorem 4.1 in the main text]
\label{mainthmapp}
    Under assumptions 4.1-4.7 in the main text,
consider the case where $f_{\text{PT}}$ and 
    $f_{\text{FT}}$ share the same architecture. 
For any $\gamma > 0$, there exists $M \in \mathbb{N}$, a full connected neural network with width and hidden dimension $n \geq M$, the following holds with probability at least $1- \gamma - c_1 \exp \left(-c_2 N \widehat{\varrho}_N^2\right)$ over pre-trained initialization when applying gradient descent with learning rate $\epsilon = \frac{\eta_0}{m}$, and $\eta_0 $ for standard parameterization and NTK parameterization, respectively 

\begin{equation}
     \|f_{\widehat{T}_{op}} - f_{\text{tgt}}\|_2^2       \leq \mathcal{O}\left( N_2^{-\frac{1}{2}}\right).
\end{equation}
where $c_1, c_2$ are some universal positive constants.
\end{theorem}
\begin{proof}
This can be proved based on the proved NTK convergence result (Theorem \ref{suppntkconver} )
and network linearlization result under the pre-trained initialization (Lemma \ref{supplinlemma}). See detailed steps in the proof of  Corollary 1   in \cite{sunarxiv}.
\end{proof}

\begin{theorem}[Theorem 5.2 in the main text]
\label{suppeasiythm}
    Under assumptions 4.1-4.7 in the main text,
    consider the case when we append a linear head to the hidden layer of pretrained network. 
for any $\gamma > 0$, there exists $M \in \mathbb{N}$, a full connected neural network with width  and hidden dimension $n \geq M$, the following holds with probability at least $1- \gamma - c_3 \exp \left(-c_4 N \widehat{\varrho}_N^2\right)$ over warm-start initialization when applying gradient descent with learning rate $\epsilon = \frac{\eta_0}{m}$, and $\eta_0 $ for standard parameterization and NTK parameterization, respectively 
\label{cor1}
\begin{equation}
     \|f_{\widehat{T}_{op}} - f_{\text{tgt}}\|_2^2       \leq \mathcal{O}\left( N_2^{-\frac{1}{2}}\right).
\end{equation}
where $c_3, c_4$ are some universal positive constants.
\end{theorem}
\begin{proof}
The proof strategy is the same as that of  Theorem \ref{mainthmapp}.      
\end{proof}

\section{Task Arithmetic}

\subsection{Proof of Corollary 5.1.1}
Apply Lemma \ref{supplinlemma}, consider the operation of negating a task vector $\tau$, for the resulting network and $f_{\text{FT}}$, we have
\begin{equation}
    f_{\text{neg}} =  f_{\text{PT}} - \tau \nabla f_{\text{PT}}(\theta)
    + e_1
    , \quad    f_{\text{FT}} \approx f_{\text{PT}}  + \tau \nabla f_{\text{PT}}(\theta) + e_2
\end{equation}
Consider evaluate $f_{\text{neg}}$
and $f_{\text{FT}}$
  on test sets using the MSE loss. We can compute
  $MSE(f_{\text{neg}}) - 
 MSE(f_{\text{FT}})$:
  \begin{equation}
 4 \left(Y^{\text{test}} - f_{\text{PT}}(X^{\text{test}})\right)^{\top} \tau \nabla f_{\text{PT}}(\theta,
 X^{\text{test}}
 ).
  \end{equation}
Then apply Theorem \ref{suppeasiythm}, we can approximate  $Y^{\text{test}}$ as 
\begin{equation}
    Y^{\text{test}}  
      =  f_{\text{PT}}(X^{\text{test}})  + \tau \nabla f_{\text{PT}}(\theta,
 X^{\text{test}}) + w^{\text{test}} + e_3
\end{equation}
Then the expectation of the above difference, we have
\begin{equation}
\begin{aligned}
 \mathbb{E}\left[  MSE(f_{\text{neg}}) - 
 MSE(f_{\text{FT}}) \right] = 
    \mathbb{E}\left[ 
 4  {\nabla f_{\text{PT}}(\theta,
 X^{\text{test}})}^{\top} \tau^{\top} \tau \nabla f_{\text{PT}}(\theta,
 X^{\text{test}}) \right] + \sum_{i=1}^3\mathbb{E} e_i
\end{aligned}
\end{equation}
Since the error of $e_i$ can be controlled by $\tilde{\mathcal{O}}\left( \frac{1}{{n}^{1/2}}\right)$ with sufficiently large width and ${\nabla f_{\text{PT}}(\theta,
 X^{\text{test}})}$ converges to some stationary Jacobian in probability, we can conclude that w.h.p. 
 $\mathbb{E}\left[  MSE(f_{\text{neg}}) - 
 MSE(f_{\text{FT}}) \right] >0$
when the width and hidden dimension $n$ is sufficiently large. 
\subsection{Proof of Corollary 5.1.2}

For the corresponding  fine-tuning model, apply Lemma \ref{supplinlemma}, we have
\begin{equation}
\begin{aligned}
     f_{\text{FT}_1} & = f_{\text{PT}} + 
     \underbrace{
     \sum_{\mathbf{X}_i \in 
    \boldsymbol{X}^{\text{task}_1}
    } \gamma^{t}_{i} \mathbb{K}(\cdot, \mathbf{X}_i)}_{g_1} + e_1 \\
      f_{\text{FT}_2} &= f_{\text{PT}} + 
     \underbrace{ 
      \sum_{\mathbf{X}_i \in 
    \boldsymbol{X}^{\text{task}_2}
    } \lambda^{t}_{i} \mathbb{K}(\cdot, \mathbf{X}_i)}_{g_2} + e_2
\end{aligned}
\end{equation}
And this indicates the multi task model obtained by applying task vector can be written as 
\begin{equation}
    f_{\text{MultiTask}} = f_{\text{PT}}
    +
    g_1 + g_2  + e_1  +e_2
\end{equation}
According to assumption 5.1 in the main text and the reproducing property of RKHS, for any $\mathbf{X}_i \in X^{\text{task}_1},  \mathbf{X}_j \in X^{\text{task}_2}$
\begin{equation}
    \mathbb{K}( \mathbf{X}_i, \mathbf{X}_j ) = 0.
\end{equation}
And since error terms $e_i$ can be controlled for sufficiently large $n$, we can derive the desire high probability estimation results.

\end{document}